\documentclass[10pt,twocolumn,letterpaper]{article}

\usepackage{cvpr}
\usepackage{times}
\usepackage{epsfig}
\usepackage{graphicx}
\usepackage{amsmath}
\usepackage{amssymb}
\usepackage{amsthm}
\usepackage{wrapfig}
\usepackage{epstopdf}

\usepackage[american]{babel}
\usepackage{microtype}
\usepackage{balance}

\usepackage{enumitem}
\DeclareMathOperator*{\argmin}{arg\,min}

\newtheorem{defi}{Definition}
\newtheorem{thm}{Theorem}
\newtheorem{cor}{Corollary}

\usepackage{algorithm}
\usepackage{algorithmic}
\usepackage[ruled,vlined,algo2e]{algorithm2e}
\providecommand{\SetAlgoLined}{\SetLine}

\usepackage[pagebackref=true,breaklinks=true,letterpaper=true,colorlinks,bookmarks=false]{hyperref}

\cvprfinalcopy 

\def\cvprPaperID{170} 

\ifcvprfinal\fi
\begin{document}

\title{BPGrad: Towards Global Optimality in Deep Learning via Branch and Pruning}

\author{Ziming Zhang\thanks{Joint first authors for the paper.}\\
		Mitsubishi Electric Research Laboratories \\
		201 Broadway, Cambridge, MA 02139-1955 \\
		{\tt\small zzhang@merl.com}
		\and
		Yuanwei Wu\footnotemark[1], Guanghui Wang\\
		EECS, The University of Kansas \\
		1450 Jayhawk Blvd., Lawrence, KS 66045\\
		{\tt\small \{y262w558, ghwang\}@ku.edu}		
}

\maketitle

\begin{abstract}
Understanding the global optimality in deep learning (DL) has been attracting more and more attention recently. Conventional DL solvers, however, have not been developed intentionally to seek for such global optimality. In this paper we propose a novel approximation algorithm, {\em BPGrad}, towards optimizing deep models globally via branch and pruning. Our BPGrad algorithm is based on the assumption of Lipschitz continuity in DL, and as a result it can adaptively determine the step size for current gradient given the history of previous updates, wherein theoretically no smaller steps can achieve the global optimality. We prove that, by repeating such branch-and-pruning procedure, we can locate the global optimality within finite iterations. Empirically an efficient solver based on BPGrad for DL is proposed as well, and it outperforms conventional DL solvers such as Adagrad, Adadelta, RMSProp, and Adam in the tasks of object recognition, detection, and segmentation.
	
\end{abstract}

\section{Introduction}
Deep learning (DL) has been demonstrated successfully in many different research areas such as image classification~\cite{krizhevsky2012imagenet}, speech recognition \cite{hinton2012deep} and natural language processing \cite{sutskever2014sequence}. In general, its empirical success stems mainly from better network architectures \cite{he2016deep}, larger mount of training data \cite{imagenet_cvpr09}, and better learning algorithms \cite{goyal2017accurate}.

However, theoretical understanding of DL for its success in applications still remains elusive. Very recently researchers start to understand DL from the perspective of optimization such as the optimality of the learned models \cite{haeffele2017global, hand2017global, yun2017global}. It has been proved that under certain (very restrictive) conditions the critical points learned for the deep models actually achieve global optimality, even though the optimization in deep learning is highly nonconvex. These theoretical results may partially explain why such deep models work well in practice. 

Global optimality is always desirable and preferred in optimization. Locating global optimality in deep learning, however, is extremely challenging due to its high non-convexity, and thus no conventional DL solvers, \eg stochastic gradient descent (SGD) \cite{bottou2016optimization}, Adagrad~\cite{duchi2011adaptive}, Adadelta~\cite{zeiler2012adadelta}, RMSProp~\cite{Tieleman2012} and Adam~\cite{kingma2014adam}, is intentionally developed for this purpose, to our best knowledge. Alternatively different regularization techniques are applied to smooth the objective functions in DL so that the solvers can converge to some geometrically wider and flatter regions in the parameter space where good model solutions may exist \cite{zhang2015deep,chaudhari2016entropy,zhang2017Convergent}. But these solutions may not necessarily be the global optimum.

\begin{figure}[t]
	\begin{center}
		\includegraphics[width=.61\linewidth]{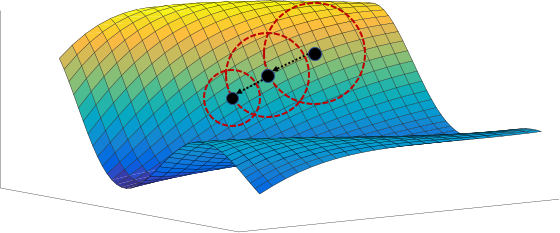}
		\caption{\footnotesize Illustration of how BPGrad works, where each black dot denotes the solution at each iterations (\ie branch), directed dotted lines denote the current gradients, and red dotted circles denote the regions wherein there should be no solutions achieving global optimality (\ie pruning). BPGrad can automatically estimate the scales of these regions based on the function evaluation of solutions and the Lipschitz continuity assumption.}
		\label{fig:BPGrad}
	\end{center}
	\vspace{-8mm}
\end{figure}

Inspired by the techniques in global optimization of nonconvex functions, we propose a novel approximation algorithm, {\em BPGrad}, which has the ability of locating global optimality in DL via branch and pruning (BP). BP \cite{sotiropoulos2001branch} is a well-known algorithm developed for searching for global solutions for nonconvex optimization problems. Its basic idea is to effectively and gradually shrink the gap between the lower and upper bounds of global optimum by efficiently branching and pruning the parameter space. Fig.~\ref{fig:BPGrad} illustrates the optimization procedure in BPGrad.

In order to branch and prune the space we assume that the objective functions in DL are Lipschitz continuous \cite{erikssonapplied}, or can be approximated by Lipschitz functions. 
This is motivated by the facts that (1) Lipschitz continuity provides a natural way to estimate the lower and upper bounds of the global optimum (see Sec. \ref{sssec:lub}) used in BP, and (2) it can also serve as regularization, if needed, to smoothen the objective functions so that the returned solutions can generalize well. 

\begin{wrapfigure}{r}{.43\linewidth}
	\begin{center}
		\includegraphics[width=\linewidth]{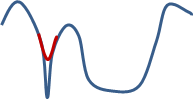}
		\vspace{-2mm}
		\caption{\footnotesize Illustration of Lipschitz continuity as regularization (red) to smoothen a function (blue).}
		\label{fig:Lipschitz}
	\end{center}
	\vspace{-12pt}
\end{wrapfigure}
In Fig. \ref{fig:Lipschitz} we illustrate the functionality of Lipschitz continuity as regularization, where the noisy narrower but deeper valley is smoothed out, while the wider but shallower valley is preserved. 
Such regularization behavior can prevent algorithms from being stuck in bad local minima. Also this is advocated and demonstrated to be crucial in order to achieve good generalization of learned DL models in several recent works such as~\cite{chaudhari2016entropy}. {\em In this sense, our BPGrad algorithm/solver essentially aims to locate global optimality in the smoothed objective functions for DL.}

Further BPGrad can generate solutions along the directions of gradients (\ie branch) based on the estimated regions wherein no global optimum should exist theoretically (\ie pruning), and by repeating such branch-and-pruning procedure BPGrad can locate global optimum. Empirically the high demand of computation as well as footprint in memory for running BPGrad inspires us to develop an efficient DL solver to approximate BPGrad towards global optimization.

\noindent
{\bf Contributions:}
The main contributions of our work are:
\begin{enumerate}[noitemsep]
\item[{\em C1.}] We propose a novel approximation algorithm, BPGrad, which is intent on locating global optimum in DL. To our best knowledge, our approach is the {\em first} algorithmic attempt towards global optimization in DL. 
\item[{\em C2.}] Theoretically we prove that BPGrad can converge to global optimality within finite iterations.
\item[{\em C3.}] Empirically we propose a novel and efficient DL solver based on BPGrad to reduce the requirement of computation as well as footprint in memory. We provide both theoretical and empirical justification for our solver towards preserving the theoretical properties of BPGrad. We demonstrate that our solver outperforms conventional DL solvers in the applications of object recognition, detection, and segmentation.
\end{enumerate}

\subsection{Related Work}
\noindent
\textbf{Global Optimality in DL:} The empirical loss minimization problem in learning deep models is highly dimensional and nonconvex with potentially numerous local minima and saddle points. Blum and Rivest~\cite{blum1989training} showed that it is difficult to find the global optima because  in the worst case even learning a simple 3-node neural network is NP-complete. 

In spite of the difficulties in optimizing deep models, researchers have attempted to provide empirical as well as theoretical justification for the success of these models \wrt global optimality in learning. Zhang \etal~\cite{zhang2016understanding} empirically demonstrated that sufficiently over-parametrized networks trained with stochastic gradient descent can reach global optimality. Choromanska \etal~\cite{choromanska2015loss} studied the loss surface of multilayer networks using spin-glass model and showed that for many large-size decoupled networks, there exists a band with many local optima, whose objective values are small and close to that of a global optimum. Brutzkus and Globerson~\cite{brutzkus2017globally} showed that gradient descent converges to the global optimum in polynomial time on a shallow neural network with one hidden layer and a convolutional structure and a ReLU activation function. Kawaguchi~\cite{kawaguchi2016deep} proved that the error landscape does not have bad local minima in the optimization of linear deep neural networks. Yun \etal~\cite{yun2017global} extended these results and proposed sufficient and necessary conditions for a critical point to be a global minimum. Haeffele and Vidal~\cite{haeffele2017global} suggested that it is critical to balance the degrees of positive homogeneity between the network mapping and the regularization function to prevent non-optimal local minima in the loss surface of neural networks. Nguyen and Hein~\cite{nguyen2017loss} argued that almost all local minima are global optimal in fully connected wide neural networks, whose number of hidden neurons of one layer is larger than that of training points. Soudry and Carmon~\cite{soudry2016no} employed smoothed analysis techniques to provide theoretical guarantee that the highly nonconvex loss functions in multilayer networks can be easily optimized using local gradient descent updates. Hand and Voroninski \cite{hand2017global} provided theoretical properties for the problem of enforcing priors provided by generative deep neural networks via empirical risk minimization by establishing the favorable global geometry.

\noindent
\textbf{DL Solvers:} SGD \cite{bottou2016optimization} is the most widely used DL solver due to its simplicity, whose learning rate (\ie, step size for gradient) is predefined. In general, SGD suffers from slow convergence, and thus its learning rate needs to be carefully tuned. To improve the efficiency of SGD, several DL solvers with adaptive learning rates have been proposed, including Adagrad~\cite{duchi2011adaptive}, Adadelta~\cite{zeiler2012adadelta}, RMSProp~\cite{Tieleman2012} and Adam~\cite{kingma2014adam}. These solvers integrate the advantages from both stochastic and batch methods where small mini-batches are used to estimate diagonal second-order information heuristically. These solvers have the capability of escaping saddle points and often yield faster convergence empirically. 

Specifically, Adagrad is well suited for dealing with sparse data, as it adapts the learning rate to the parameters, performing smaller updates on frequent parameters and larger updates on infrequent parameters. However, it suffers from shrinking on the learning rate, which motivates Adadelta, RMSProp and Adam. Adadelta accumulates squared gradients to be fixed values rather than over time in Adagrad, RMSProp updates the parameters based on the rescaled gradients, and Adam does so based on the estimated mean and variance of the gradients. Very recently, Mukkamala \etal ~\cite{mukkamala2017variants} proposed variants of RMSProp and Adagrad with logarithmic regret bounds. 

\noindent
{\bf Convention \vs Ours:} Though the properties of global optimality in DL are very attractive, as far as we know, however, there is no solver developed intentionally to capture such global optimality so far. To fill this void, we propose our BPGrad algorithm towards global optimization in DL.

From the optimization perspective, our algorithm shares similarities with the recent work \cite{MalherbeICML17} on global optimization of general Lipschitz functions (not specifically for DL). In~\cite{MalherbeICML17} a uniform sampler is utilized to maximize the lower bound of the maximizer (equivalently minimizing the upper bound of the minimizer) subject to Lipschitz conditions. Convergence properties {\em w.h.p.} are derived. In contrast, our approach considers estimating both lower and upper bounds of the global optimum, and employs the gradients as guidance to more effectively sample the parameter space for pruning. Convergence is proved to show that our algorithm will terminate within finite iterations.

From the empirical solver perspective, our solver shares similarities with the recent work \cite{koushik2016improving} on improving SGD using the feedback from the objective function. Specifically~\cite{koushik2016improving} tracks the relative changes in the objective function with a running average, and uses it to adaptively tune the learning rate in SGD. No theoretical analysis, however, is provided for justification. In contrast, our solver does use the feedback from the object function to determine the learning rate adaptively but based on the rescaled distance between the feedback and the current lower bound estimation. Theoretical as well as empirical justifications are established.

\section{BPGrad Algorithm for Deep Learning}\label{ssec:BnB_LF}
\subsection{Key Notation} 

We denote $\mathbf{x}\in\mathcal{X}\subseteq\mathbb{R}^d$ as the parameters in the neural network, $(\omega, y)\in \Omega \times \mathcal{Y}$ as a pair of a data sample $\omega$ and its associated label $y$, $\phi: \Omega\times \mathcal{X} \rightarrow \mathcal{Y}$ as the nonconvex prediction function represented by the network, $f$ as the objective function for training the network with Lipschitz constant $L\geq0$, $\nabla f$ as the gradient of $f$ over parameters $\mathbf{x}$\footnote{We assume $\nabla f\neq \mathbf{0}$ {\em w.l.o.g.} Empirically we can randomly sample a non-zero direction for update wherever $\nabla f=\mathbf{0}$.}, $\nabla \tilde{f}=\frac{\nabla f}{\|\nabla f\|_2}$ denotes the {\em normalized} gradient (\ie direction of the gradient), $f^*$ as the global minimum, and $\|\cdot\|_2$ as the $\ell_2$-norm operator over vectors.
\begin{defi}[Lipschitz Continuity \cite{erikssonapplied}]\label{def:Lipschitz}
A function $f$ is {\em Lipschitz continuous} with Lipschitz constant $L$ on $\mathcal{X}$, if there is a (necessarily nonnegative) constant $L$ such that
\begin{align}\label{eqn:LF}
|f(\mathbf{x}_1)-f(\mathbf{x}_2)|\leq L\|\mathbf{x}_1-\mathbf{x}_2\|_2, \forall \mathbf{x}_1 , \mathbf{x}_2 \in \mathcal{X}.
\end{align}
\end{defi}

\subsection{Problem Setup}
We would like to learn the parameters for a given network by minimizing the following objective function $f$:
\begin{align}\label{eqn:f}
\min_{\mathbf{x}\in\mathcal{X}} f(\mathbf{x}) \equiv \mathbb{E}_{(\omega\times y)\in\Omega\times\mathcal{Y}}\Big[\mathcal{L}(y, \phi(\omega, \mathbf{x}))\Big] + \mathcal{R}(\mathbf{x}),
\end{align}
where $\mathbb{E}$ denotes the expectation over data pairs, $\mathcal{L}$ denotes a loss function (\eg, hinge loss) for measuring the difference between the ground-truth labels and the predicted labels given data samples, and $\mathcal{R}$ denotes a regularizer over parameters. 
Particularly we assume that: 
\begin{enumerate}[noitemsep]
\item[{\em F1.}] $f$ is lower bounded by 0 and upper bounded as well, \ie $0\leq f(\mathbf{x})<+\infty, \forall \mathbf{x}\in\mathcal{X}$;
\item[{\em F2.}] $f$ is differentiable everywhere in the bounded space  $\mathcal{X}$;
\item[{\em F3.}] $f$ is Lipschitz continuous, or can be approximated by Lipschitz functions, with constant $L\geq 0$.
\end{enumerate}


\subsection{Algorithm}


\subsubsection{Lower \& Upper Bound Estimation}\label{sssec:lub}
Consider the situation where samples $\mathbf{x}_1, \cdots, \mathbf{x}_t\in\mathcal{X}$ exist for evaluation by function $f$ with Lipschitz constant $L$, whose global minimum $f^*$ is reached by the sample $\mathbf{x}^*$. 
Then based on Eq. \ref{eqn:LF} and simple algebra, we can obtain
\begin{align}\label{eqn:LUB_f}
\max_{i=1,\cdots,t}\Big\{f(\mathbf{x}_i)-L\|\mathbf{x}_i-\mathbf{x}^*\|_2\Big\}\leq f^* \leq \min_{i=1,\cdots,t}f(\mathbf{x}_i).
\end{align}
This provides us a tractable upper bound and an {\em intractable} lower bound, unfortunately, of the global minimum. The intractability comes from the fact that $\mathbf{x}^*$ is unknown, and thus makes the lower bound in Eq. \ref{eqn:LUB_f} unusable empirically.

To address this problem, we propose a novel tractable estimator, $\rho\min_{i=1,\cdots,t}f(\mathbf{x}_i)$ $(0\leq \rho <1)$. This estimator intentionally introduces a gap from the upper bound, which will be shrunk by either decreasing the upper bound or increasing $\rho$. As proved in Thm. \ref{thm:f} (see Sec. \ref{ssec:analysis}), when the parameter space $\mathcal{X}$ is fully covered by the samples $\{\mathbf{x}_i\}$, this estimator will become the lower bound of $f^*$.

In summary, we define our lower and upper bound estimators for the global minimum as $\rho\min_{i=1,\cdots,t}f(\mathbf{x}_i)$ and $\min_{i=1,\cdots,t}f(\mathbf{x}_i)$, respectively.

\subsubsection{Branch \& Pruning}

Based on our estimators, we propose a novel approximation algorithm, BPGrad, towards global optimization in DL via branch and pruning. We show it in Alg.~\ref{alg:BnB} where the predefined constant $\epsilon\geq0$ controls the precision of the solution. 


\begin{algorithm}[t]
\SetAlgoLined
\SetKwInOut{Input}{Input}\SetKwInOut{Output}{Output}
\Input{objective function $f$ with Lipschitz constant $L\geq0$, precision $\epsilon\geq 0$}
\Output{minimizer $\mathbf{x}^*$}
\BlankLine
Randomly initialize $\mathbf{x}_1$, $t\leftarrow 1$, $\rho\leftarrow 0$;

\While{$\min_{i=1,\cdots,t}f(\mathbf{x}_i)\leq \frac{\epsilon}{1-\rho}$}{
\While{$\exists \mathbf{x}_{t+1}\in\mathcal{X}$ satisfies Eq. \ref{eqn:sampling_rule}}{
Compute $\mathbf{x}_{t+1}$ by solving Eq. \ref{eqn:GD-sampler};

$t\leftarrow t+1$;
}
Increase $\rho$ such that $0\leq\rho<1$ still holds;
}
\Return $\mathbf{x}^*=\mathbf{x}_{i^*}$ where $i^*\in\argmin_{i=1,\cdots,t} f(\mathbf{x}_i)$;
\caption{BPGrad Algorithm for Deep Learning}\label{alg:BnB}
\end{algorithm}

\noindent
{\bf Branch:} The {\em inner} loop in Alg. \ref{alg:BnB} conducts the branch operation to split the parameter space recursively by {\em sampling}. Towards this goal, we need a mapping between the parameter space and the bounds. Considering the lower bound in Eq.~\ref{eqn:LUB_f}, we propose sampling $\mathbf{x}_{t+1}\in\mathcal{X}$ based on the previous samples $\mathbf{x}_1,\cdots,\mathbf{x}_t\in\mathcal{X}$ so that it satisfies
\begin{align}\label{eqn:sampling_rule}
\hspace{-2mm}\max_{i=1,\cdots,t}\Big\{f(\mathbf{x}_i)-L\|\mathbf{x}_i-\mathbf{x}_{t+1}\|_2\Big\}\leq\rho\min_{i=1,\cdots,t}f(\mathbf{x}_i).
\end{align}
Note that an equivalent constraint has been used in \cite{MalherbeICML17}. 

To improve sampling efficiency for decreasing the objective, we propose sampling along the directions of (stochastic) gradients with small distortion. Though gradients only encode local structures of (nonconvex) functions in a high dimensional space, they are good indicators for locating local minima \cite{pmlr-v49-lee16, panageas2016gradient}. Specifically, we propose a minimization problem for generating samples:
\begin{align}\label{eqn:GD-sampler}
& \min_{\mathbf{x}_{t+1}\in\mathcal{X}, \eta_t\geq0}\left\|\mathbf{x}_{t+1} - \left(\mathbf{x}_t - \eta_t\nabla\tilde{f}(\mathbf{x}_t)\right)\right\|_2^2 + \gamma\eta_t^2, \\
& \mbox{s.t.} \; \max_{i=1,\cdots,t}\Big\{f(\mathbf{x}_i)-L\|\mathbf{x}_i-\mathbf{x}_{t+1}\|_2\Big\}\leq\rho\min_{i=1,\cdots,t}f(\mathbf{x}_i), \nonumber
\end{align}
where $\gamma\geq0$ is a predefine constant controlling the trade-off between the distortion and the step size $\eta_t\geq0$. That is, under the condition in Eq. \ref{eqn:sampling_rule}, the objective in Eq. \ref{eqn:GD-sampler} aims to generate a sample that has small distortion from an anchor point, whose step size is small as well due to the locality property of gradients, along the direction of the gradient. 

Note that other reasonable objective functions may also be utilized here for sampling purpose as long as the condition in Eq. \ref{eqn:sampling_rule} is satisfied. More efficient sampling objectives will be investigated in our future work.

\noindent
{\bf Pruning:} In fact Eq. \ref{eqn:sampling_rule} specifies that new samples should be generated outside the union of a set of balls defined by previous samples. To precisely describe this requirement, we introduce a new concept of removable solution space in our work as follows:
\begin{defi}[Removable Parameter Space (RPS)]\label{def:RPS}
We define the RPS, denoted as $\mathcal{X}_{R}$, as
\begin{align}\label{eqn:ball}
\mathcal{X}_R(t) \stackrel{\mbox{\em def}}{=} \cup_{j=1,\cdots,t} \mathcal{B}\left(\mathbf{x}_j, r_j\right),
\end{align}
where $\mathcal{B}(\mathbf{x}_j, r_j)=\{\mathbf{x}\mid \|\mathbf{x} - \mathbf{x}_j\|_2 < r_j, \mathbf{x}\in\mathcal{X}\}, \forall j$ defines a ball centered at sample $\mathbf{x}_j\in\mathcal{X}$ with radius $r_j=\frac{1}{L}\left[f(\mathbf{x}_j)-\rho\min_{i=1,\cdots,t}f(\mathbf{x}_i)\right], \forall j$.
\end{defi}
RPS specifies a region wherein the function evaluations of all the points cannot be smaller than the lower bound estimator conditioning on the Lipschitz continuity assumption. Therefore, when the lower bound estimator is higher than the global minimum $f^*$, we can safely remove all the points in RPS without evaluation. However, when it becomes smaller than $f^*$, we risk missing the global solutions. 

To address this issue, we propose the {\em outer} loop in Alg. \ref{alg:BnB} to increase the lower bound for drawing more samples which may further decrease the upper bound later.

\subsection{Theoretical Analysis}\label{ssec:analysis}
\begin{thm}[Lower \& Upper Bounds]\label{thm:f}
Whenever $\mathcal{X}_R(t)\equiv\mathcal{X}$ holds, the samples generated by Alg. \ref{alg:BnB} satisfies
\begin{align}\label{eqn:rho}
\rho\min_{i=1,\cdots,t}f(\mathbf{x}_i) \leq f^* \leq \min_{i=1,\cdots,t}f(\mathbf{x}_i).
\end{align}
\end{thm}
\begin{proof}
Since $f^*$ is the global minimum, it always holds that $f^* \leq \min_{i=1,\cdots,t}f(\mathbf{x}_i)$. Now suppose that if $\mathcal{X}_R(t)\equiv\mathcal{X}$ holds, $\rho\min_{i=1,\cdots,T}f(\mathbf{x}_i) > f^*$ holds as well. Then there would exist at least one point (\ie global minimum) left for sampling, contradicting the condition of $\mathcal{X}_R(t)\equiv\mathcal{X}$. We then complete the proof.
\end{proof}

\begin{cor}[Approximation Error Bound]
Whenever both $\min_{i=1,\cdots,t}f(\mathbf{x}_i)\leq \frac{\epsilon}{1-\rho}$ and $\mathcal{X}_R(t)\equiv\mathcal{X}$ hold, it is satisfied that
\begin{align}
\min_{i=1,\cdots,t}f(\mathbf{x}_i) - f^* \leq \epsilon.
\end{align}
\end{cor}



\begin{thm}[Convergence within Finite Samples]\label{thm:number}
The total number of samples, $T$, in Alg. \ref{alg:BnB} is upper bounded by:
\begin{align}
T \leq \left[\frac{2L}{(1-\rho)f_{\min}}\right]^d \cdot \frac{V_{\mathcal{X}}}{C},
\end{align}
where $V_{\mathcal{X}}$ denotes the volume of the space $\mathcal{X}$, $C=\frac{\pi^{\frac{d}{2}}}{\Gamma\left(\frac{d}{2}+1\right)}$ denotes a constant, and $f_{\min} = \min_{i=1,\cdots,T}f(\mathbf{x}_i)$ denotes the minimum evaluation.
\end{thm}
\begin{proof} 
Given $\forall j, \forall t$ such that $1\leq j\leq t\leq T-1$, we have
\begin{align}
\hspace{-5mm}\|\mathbf{x}_{t+1}-\mathbf{x}_j\|_2 & \geq \frac{1}{L}\left[f(\mathbf{x}_j) - \rho\min_{i=1,\cdots,t}f(\mathbf{x}_i)\right]  \\
\hspace{-5mm} & \geq \frac{1-\rho}{L}\cdot\min_{i=1,\cdots,t}f(\mathbf{x}_i) \geq \frac{(1-\rho)f_{\min}}{L}. \nonumber
\end{align}
This allows us to generate two balls $\mathcal{B}\left(\mathbf{x}_{t+1}, \frac{(1-\rho)f_{\min}}{2L}\right)$ and $\mathcal{B}\left(\mathbf{x}_j, \frac{(1-\rho)f_{\min}}{2L}\right)$ so that they have no overlap with each other. As a result we can generate $T$ balls with radius of $\frac{(1-\rho)f_{\min}}{2L}$ and no overlaps, and their accumulated volume should be no bigger than $V_{\mathcal{X}}$. That is,
\begin{align}
V_{\mathcal{X}} \geq \sum_{t=1}^T V_{\mathcal{B}\left(\mathbf{x}_t, \frac{(1-\rho)f_{\min}}{2L}\right)} = C\left[\frac{(1-\rho)f_{\min}}{2L}\right]^d T.
\end{align}
Further using simple algebra we can complete the proof.
\end{proof}

\section{Approximate DL Solver based on BPGrad}

\begin{algorithm}[t]
\SetAlgoLined
\SetKwInOut{Input}{Input}\SetKwInOut{Output}{Output}
\Input{number of evaluations $n$ repeating $N$ times at most, objective function $f$ with Lipschitz constant $L\geq0$, momentum $0\leq\mu\leq 1$}
\Output{minimizer $\mathbf{x}^*$}
\BlankLine
$t\leftarrow 1, \mathbf{v}_1\leftarrow\mathbf{0}$, and randomly initialize $\mathbf{x}_1$;

\For{$m\leftarrow 1$ \KwTo $N$}{
$\rho\leftarrow 1 - \frac{1}{m}$;

\While{$t < mn$}{

$\mathbf{v}_{t+1}\leftarrow\mu\mathbf{v}_t - \frac{f(\mathbf{x}_t) - \rho\min_{i=1,\cdots,t}f(\mathbf{x}_i)}{L}\cdot\frac{\nabla f(\mathbf{x}_t)}{\|\nabla f(\mathbf{x}_t)\|_2}$;

$\mathbf{x}_{t+1}\leftarrow\mathbf{x}_{t} + \mathbf{v}_{t+1}$;


$t\leftarrow t+1$;
}
\lIf{$\min_{i=1,\cdots,t}f(\mathbf{x}_i)\leq \frac{\epsilon}{1-\rho}$ holds}{
Break
}
}
\Return $\mathbf{x}^*=\mathbf{x}_{i^*}$ where $i^*\in\argmin_{i=1,\cdots,n} f(\mathbf{x}_i)$;
\caption{BPGrad based Solver for Deep Learning}\label{alg:BnB-GD}
\end{algorithm}

Though the BPGrad algorithm has nice theoretical properties for global optimization, directly applying Alg. \ref{alg:BnB} to deep learning will incur the following problems that limit its empirical usage:
\begin{enumerate}[noitemsep]
\item[{\em P1.}] From Thm. \ref{thm:number} we can see that due to the high dimensionality of the parameter space in DL it is impractical to draw sufficient samples to cover the entire space. 
\item[{\em P2.}] Solving Eq. \ref{eqn:GD-sampler} involves the knowledge of previous samples, which incurs significant amount of both computational and storage burden for deep learning. 
\item[{\em P3.}] Computing $f(\mathbf{x}_t)$ and $\nabla\tilde{f}(\mathbf{x}_t), \forall \mathbf{x}_t\in\mathcal{X}$ is time-consuming, especially for large-scale data. 
\end{enumerate}

To address problem {\em P1}, in practice we manually set the maximum iterations for both inner and outer loops in Alg. \ref{alg:BnB}.

To address problem {\em P2}, we further make some extra assumptions to simplify the branching/sampling procedure based on Eq. \ref{eqn:GD-sampler} as follows:

\begin{enumerate}[noitemsep]
\item[{\em A1.}] Minimizing distortion is much important than minimizing step sizes, \ie $\gamma\ll1$;  
\item[{\em A2.}] $\mathcal{X}$ is sufficiently large where $\exists\eta_t\geq0$ so that $\mathbf{x}_{t+1}=\mathbf{x}_t-\eta_t\nabla\tilde{f}(\mathbf{x}_t)\in\mathcal{X}\setminus\mathcal{X}_R(t)$ always holds;
\item[{\em A3.}] $\eta_t\geq0$ is always sufficiently small for local update.
\item[{\em A4.}] $\mathbf{x}_{t+1}$ can be sampled only based on $\mathbf{x}_t$ and $\nabla\tilde{f}(\mathbf{x}_t)$.
\end{enumerate}
By imposing these assumptions upon Eq. \ref{eqn:GD-sampler}, we can directly compute the solution as follows:
\begin{align}\label{eqn:eta_t}
\eta_t = \frac{1}{L}\left[f(\mathbf{x}_t) - \rho\min_{i=1,\cdots,t}f(\mathbf{x}_i)\right].
\end{align}

To address problem {\em P3}, we utilize mini-batches to estimate $f(\mathbf{x}_t)$ and $\nabla\tilde{f}(\mathbf{x}_t)$ efficiently in each iteration.

In summary, we list our BPGrad solver in Alg.~\ref{alg:BnB-GD} by modifying Alg. \ref{alg:BnB} for the sake of fast sampling as well as low memory footprint in DL, but at the risk of being stuck in local regions. Fig. \ref{fig:sample} illustrates such scenarios in a 1D example. In (b) the sampling method falls into a loop because it does not consider the history of samples but only current one. In contrast, the sampling method in (a) is able to keep generating new samples by avoiding the RPS of previous samples with more computation and storage, as expected.

\begin{figure}[t]
	\begin{minipage}[b]{0.495\columnwidth}
		\begin{center}
			\centerline{\includegraphics[width=\columnwidth]{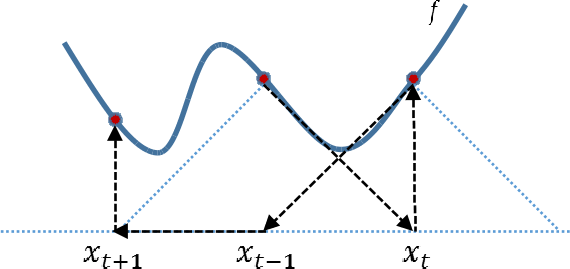}}
			\centerline{\footnotesize{(a) Sampling using Eq. \ref{eqn:GD-sampler}}}
		\end{center}
	\end{minipage}
	\begin{minipage}[b]{0.495\columnwidth}
		\begin{center}
			\centerline{\includegraphics[width=\columnwidth]{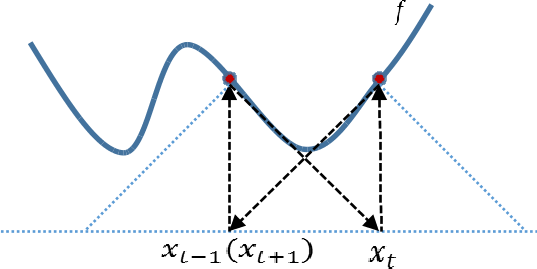}}
			\centerline{\footnotesize{(b) Sampling using Eq. \ref{eqn:eta_t}}}
		\end{center}
	\end{minipage}	
    \vspace{-7mm}
	\caption{\footnotesize{1D illustration of difference in sampling between {\bf (a)} using Eq.~\ref{eqn:GD-sampler} and {\bf (b)} using Eq. \ref{eqn:eta_t}. Here the solid blue lines denote function $f$, the black dotted lines denote the sampling paths starting from $\mathbf{x}_{t-1} \rightarrow \mathbf{x}_t \rightarrow \mathbf{x}_{t+1}$, and each big triangle surrounded by blue dotted lines denotes the RPS of each sample. As we see, (b) suffers from being stuck locally, while (a) can avoid the locality based on the RPS.
	}}\label{fig:sample}
    \vspace{-3mm}
\end{figure}

\subsection{Theoretical Analysis}
\begin{thm}[Global Property Preservation]\label{thm:fast_sample}
Let $\mathbf{x}_{t+1} = \mathbf{x}_t - \eta_t\nabla\tilde{f}(\mathbf{x}_t)$ where $\eta_t$ is computed using Eq. \ref{eqn:eta_t}. Then $\mathbf{x}_{t+1}$ satisfies Eq. \ref{eqn:sampling_rule} if it holds that
\begin{align}\label{eqn:fast_condition}
\hspace{-3mm}\left\langle\mathbf{x}_i-\mathbf{x}_t, \nabla\tilde{f}(\mathbf{x}_t)\right\rangle \geq \frac{f(\mathbf{x}_i) - f(\mathbf{x}_t)}{L}, \forall i=1,\cdots,t,
\end{align}
where $\langle\cdot,\cdot\rangle$ denotes the inner product between two vectors.
\end{thm}
\begin{proof}
Based on Eq. \ref{eqn:LF}, Eq. \ref{eqn:eta_t}, and Eq. \ref{eqn:fast_condition}, we have
\begin{align}
& \|\mathbf{x}_i-\mathbf{x}_{t+1}\|_2 \nonumber \\
& = \left(\|\mathbf{x}_i-\mathbf{x}_{t}\|_2^2+\eta_t^2+2\eta_t\left\langle\mathbf{x}_i-\mathbf{x}_t, \nabla\tilde{f}(\mathbf{x}_t)\right\rangle\right)^{\frac{1}{2}} \nonumber\\
&\geq \frac{1}{L}\left[f(\mathbf{x}_i) - \rho\min_{i=1,\cdots,t}f(\mathbf{x}_i)\right], \forall i=1,\cdots,t,
\end{align}
which is essentially equivalent to Eq. \ref{eqn:sampling_rule} based on algebra. We then can complete the proof.
\end{proof}

\begin{cor}\label{cor:1}
Suppose that a monotonically decreasing sequence $\{f(\mathbf{x}_i)\}_{i=1,\cdots,t}$ is generated to minimize function $f$ by sampling using Eq. \ref{eqn:eta_t}. Then the condition in Eq. \ref{eqn:fast_condition} can be rewritten as follows:
\begin{align}\label{eqn:cor}
\left\langle\mathbf{x}_i-\mathbf{x}_j, \nabla\tilde{f}(\mathbf{x}_j)\right\rangle \geq 0, \, 1\leq \forall i< \forall j\leq t.
\end{align}
\end{cor}

\noindent
{\bf Discussion:}
Both Thm. \ref{thm:fast_sample} and Cor. \ref{cor:1} imply that our solver prefers sampling the parameter space along a path towards a single direction, roughly speaking. However, the gradients in conventional backpropagation have little guarantee to satisfy Eq.~\ref{eqn:fast_condition} or Eq.~\ref{eqn:cor} due to lack of such constraints in learning. On the other hand, momentum~\cite{sutskever2013importance} is a well-known technique in deep learning to dampen oscillations in gradients and accelerate directions of low curvature. Therefore, our solver in Alg.~\ref{alg:BnB-GD} involves momentum to compensate such drawbacks in backpropagation for better approximation of Alg. \ref{alg:BnB}.

\subsection{Empirical Justification}

In this section we discuss the feasibility of the assumptions {\em A1-A4} for reducing computation and storage as well as preserving the properties towards global optimization in deep learning. We utilize MatConvNet \cite{vedaldi2015matconvnet} as our testbed, and run our solver in Alg. \ref{alg:BnB-GD} to train the default networks in MatConvNet for MNIST \cite{lecun1998mnist} and CIFAR10 \cite{krizhevsky2012imagenet}, respectively, using the default parameters without explicit mention. Also we set $N=1, \mu=0, L=15$ for MNIST and $L=50$ for CIFAR10 by default. For justification purpose we only run 4 epochs on each dataset, 600 and 500 iterations per epoch for MNIST and CIFAR10, respectively. For more experimental details, please refer to Sec.~\ref{sec:exp}.

Essentially assumption {\em A1} is made to support the other three to simplify the objective in Eq. \ref{eqn:GD-sampler}, and assumption {\em A2} usually holds in deep learning due to its high dimensionality. Therefore, below we only focus on empirical justification of assumptions {\em A3} and {\em A4}. 

\begin{wrapfigure}{r}{.48\linewidth}
	\vspace{-25pt}
	\begin{center}
		\includegraphics[width=\linewidth]{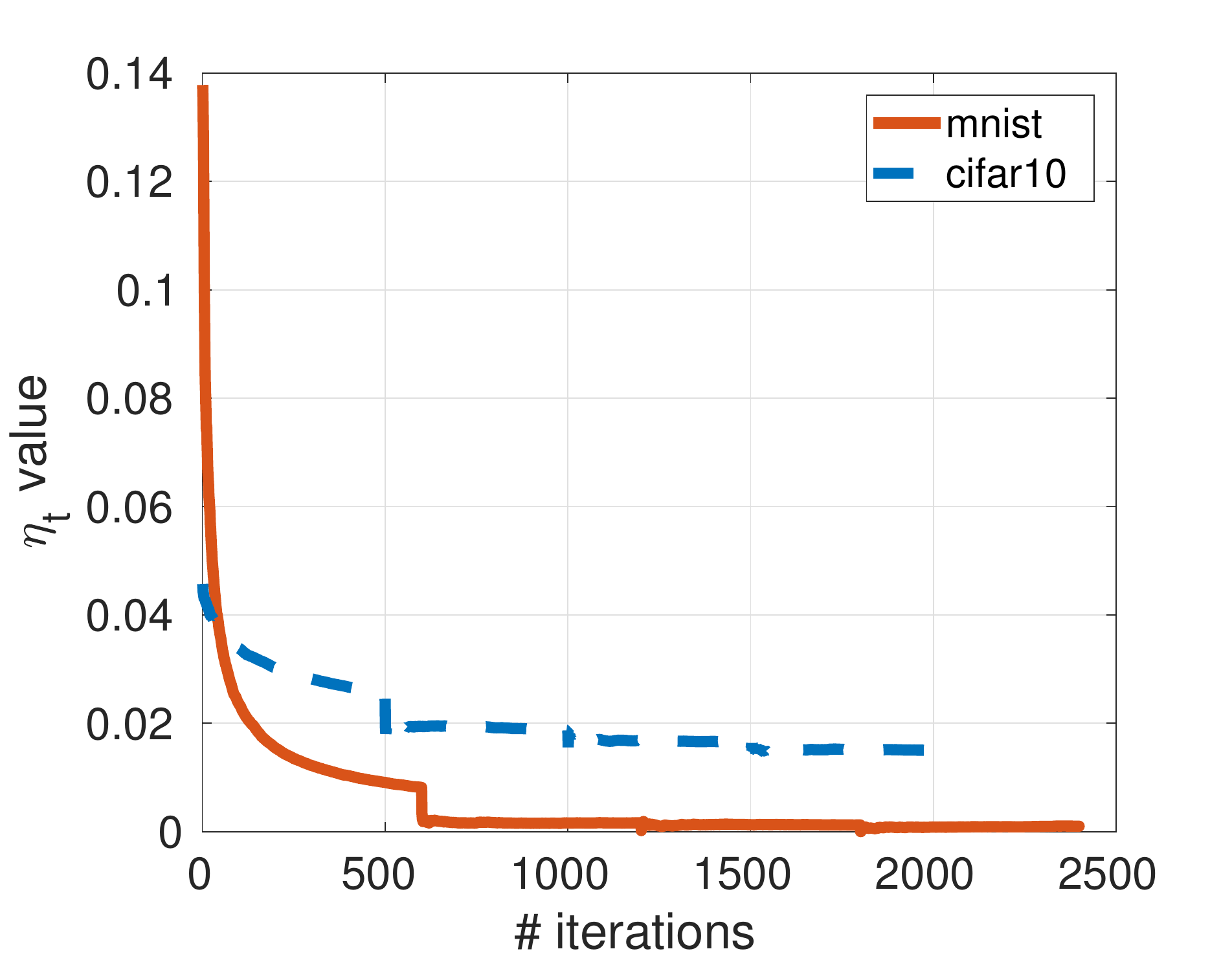}
		\caption{\footnotesize Plots of $\eta_t$ on MNIST and CIFAR10, respectively.}
		\label{fig:eta}
	\end{center}
	\vspace{-10pt}
\end{wrapfigure}

\noindent
{\bf Feasibility of {\em A3}:}
To justify this, we collect $\eta_t$'s by running Alg. \ref{alg:BnB-GD} on both datasets, and plot them in Fig.~\ref{fig:eta}. Overall these numbers are indeed sufficiently small for local update based on gradients, and $\eta_t$ decreases with the increase of iterations, in general. This behavior is expected as the objective $f$ is supposed to decrease as well \wrt iterations. The value gap at the beginning on the two datasets is induced mainly by different $L$'s. 

\begin{figure}[t]
	\begin{minipage}[b]{0.495\columnwidth}
		\begin{center}
			\centerline{\includegraphics[width=1.05\columnwidth]{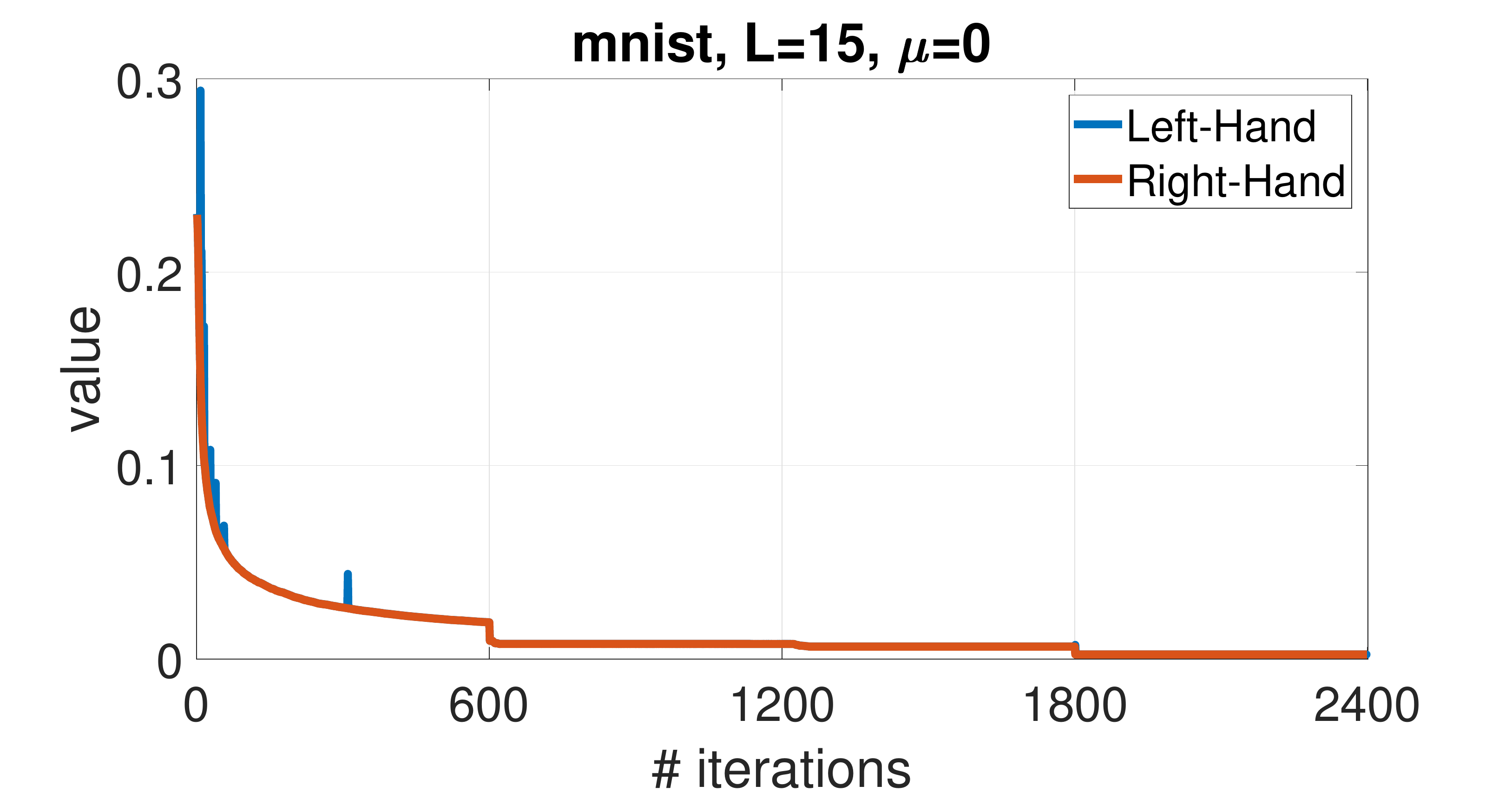}}
		\end{center}
	\end{minipage}
	\begin{minipage}[b]{0.495\columnwidth}
		\begin{center}
			\centerline{\includegraphics[width=1.05\columnwidth]{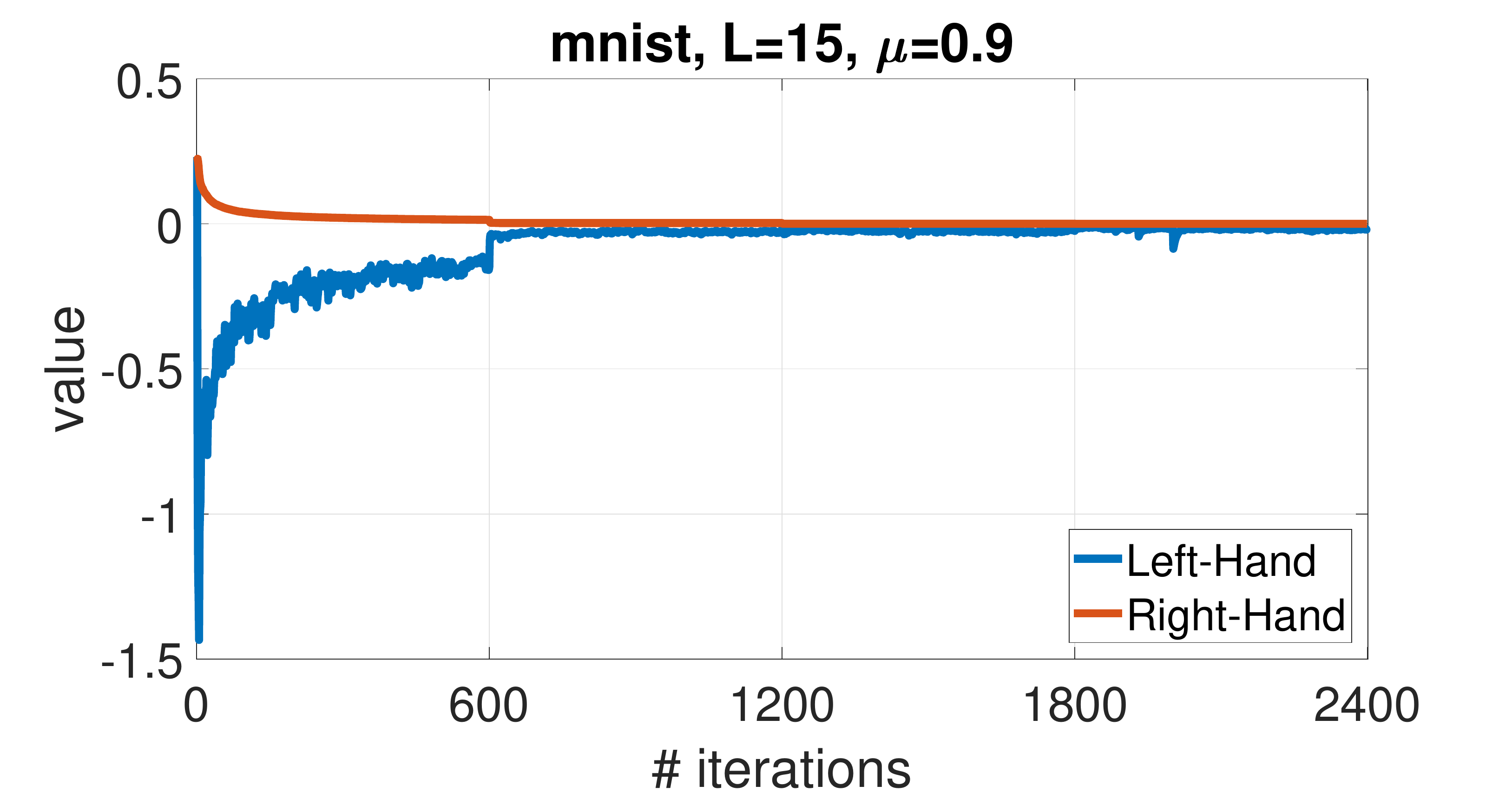}}
		\end{center}
	\end{minipage}
	\begin{minipage}[b]{0.495\columnwidth}
		\begin{center}
			\centerline{\includegraphics[width=1.05\columnwidth]{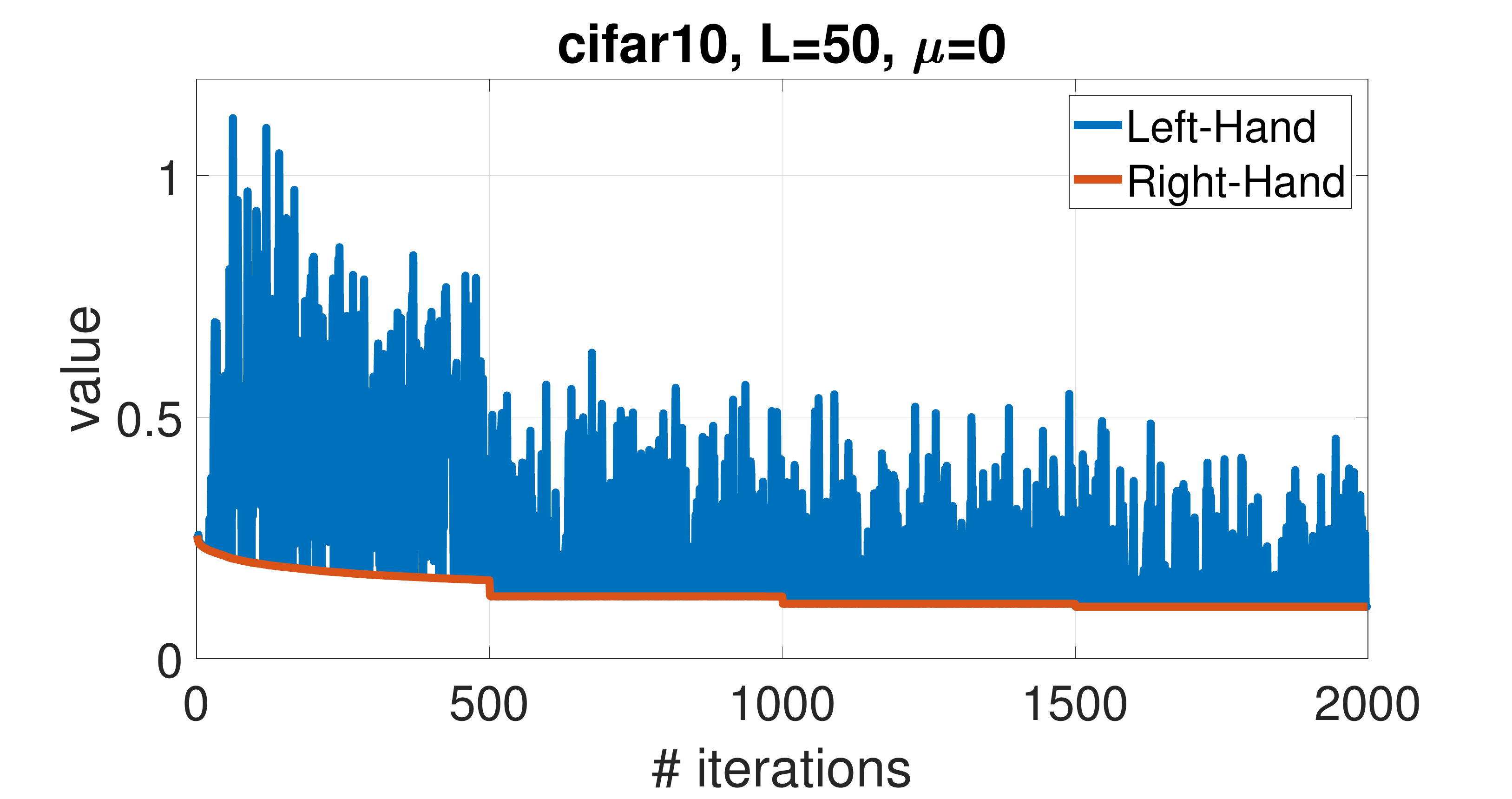}}
		\end{center}
	\end{minipage}
	\begin{minipage}[b]{0.495\columnwidth}
		\begin{center}
			\centerline{\includegraphics[width=1.05\columnwidth]{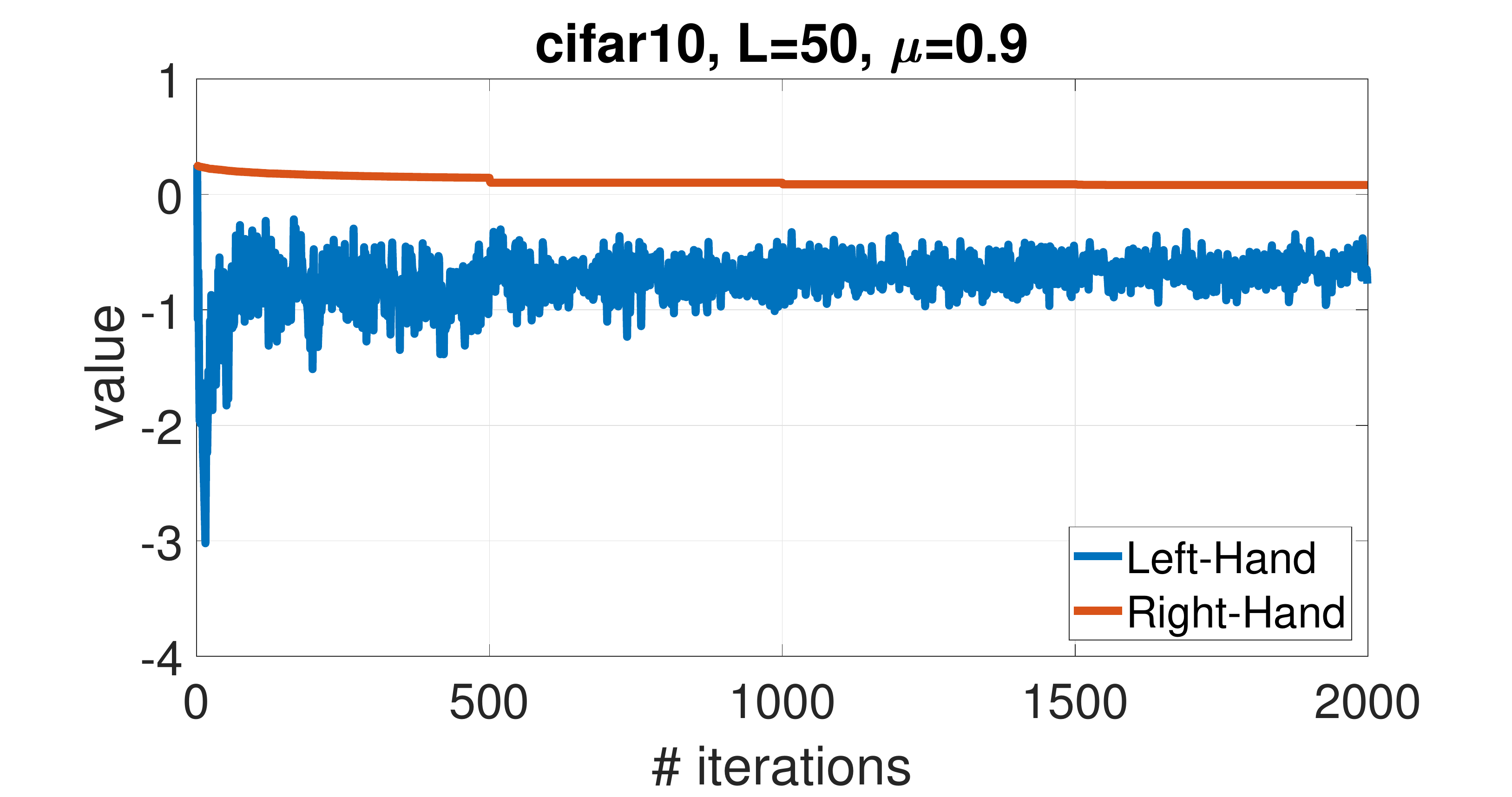}}
		\end{center}
	\end{minipage}
	\vspace{-7mm}
	\caption{\footnotesize Comparison between LHS and RHS of Eq.~\ref{eqn:sampling_rule} based on $\mathbf{x}_t$ returned by Alg.~\ref{alg:BnB-GD} using different values for momentum parameter $\mu$.}
	\label{fig:A4}
    \vspace{-3mm}
\end{figure}

\noindent
{\bf Feasibility of {\em A4}:}
To justify this, we show some evidences in Fig. \ref{fig:A4}, where we plot the left-hand side (LHS) and right-hand side (RHS) of Eq. \ref{eqn:sampling_rule} based on $\mathbf{x}_t$ returned by Alg.~\ref{alg:BnB-GD}. As we see in all the subfigures on the right with $\mu=0$ the values on RHS are always no smaller than those on LHS correspondingly. In contrast, in the remaining subfigures on the left with $\mu=0$ (\ie conventional SGD update) the values on RHS are always no bigger than those on LHS correspondingly. These observations appear to be robust across different datasets, and irrelevant to parameter $L$ which determines the radius of balls, \ie step sizes for gradients. The momentum parameter $\mu$, which is related to the directions of gradients for updating models, appear to be the only factor to make the samples of our solver satisfy Eq. \ref{eqn:sampling_rule}. This also supports our claims in Thm. \ref{thm:fast_sample} and Cor. \ref{cor:1} about the relation between model update and gradient in order to satisfy Eq.~\ref{eqn:sampling_rule}. More evidences have been provided in Sec. \ref{sssec:m&c}. Given these evidences we hypothesize that assumption {\em A4} may hold empirically when using sufficiently large values for $\mu$.

\section{Experiments}\label{sec:exp}
To demonstrate the generalization of our BPGrad solver, we test it in the applications of object recognition, detection, and segmentation by training deep convolutional neural networks (CNNs). We utilize MatConvNet as our testbed, and employ its demo code as well as default network architectures for different tasks. Since our solver has the ability of determining learning rates adaptively, we compare ours with another four widely used DL solvers with adaptive learning rates, namely Adagrad, Adadelta, RMSProp, and Adam. We tune the parameters in these solvers to achieve their best performance as we can. 

\subsection{Object Recognition}

\subsubsection{MNIST \& CIFAR10}\label{sssec:m&c}
The MNIST digital dataset consists of a training set of $60K$ images and a test set of $10K$ images in $10$ classes labeled from 0 to 9, where all images have the resolution of $28\times28$ pixels. The CIFAR-10 dataset consists of a training set of $50K$ images and a test set of $10K$ images in 10 object classes, where the image resolution is $32\times32$ pixels. 

We follow the default implementation to train an individual CNN similar to LeNet-$5$~\cite{lecun1998gradient} on each dataset. For the details of network architectures please refer to the demo code. 
Specifically for all the solvers, we train the networks for $50$ and $100$ epochs on MNIST and CIFAR10, respectively, with a mini-batch size $100$, weight decay $0.0005$, and momentum $0.9$. In addition, we fix the initial weights for two networks and the feeding order of mini-batches for fair comparison. The global learning rate is set to $0.001$ on MNIST for Adagrad, RMSProp and Adam. On CIFAR10, the global learning rate is set to $0.001$ for RMSProp, but to $0.01$ for Adagrad, Adam and Eve~\cite{koushik2016improving}, and it is reduced to $0.005$ and $0.001$ at the $31$-st and $61$-st epoch. 
Adadelta does not require the global learning rate.  

\begin{figure}[t]
	\begin{minipage}[b]{0.495\columnwidth}
		\begin{center}
			\centerline{\includegraphics[width=\columnwidth,height=0.75\columnwidth]{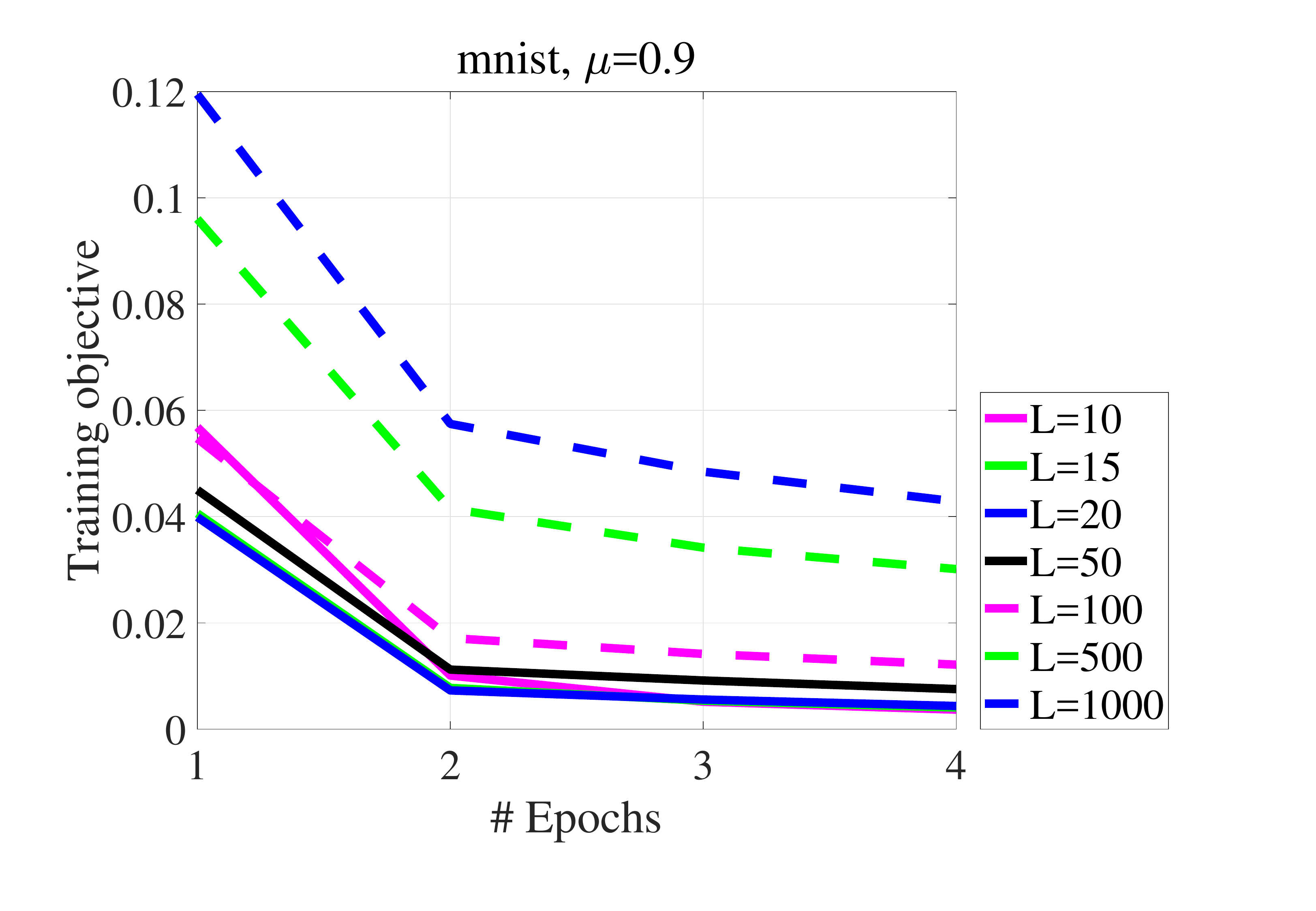}}			
		\end{center}
	\end{minipage}
	\begin{minipage}[b]{0.495\columnwidth}
		\begin{center}
			\centerline{\includegraphics[width=\columnwidth,height=0.75\columnwidth]{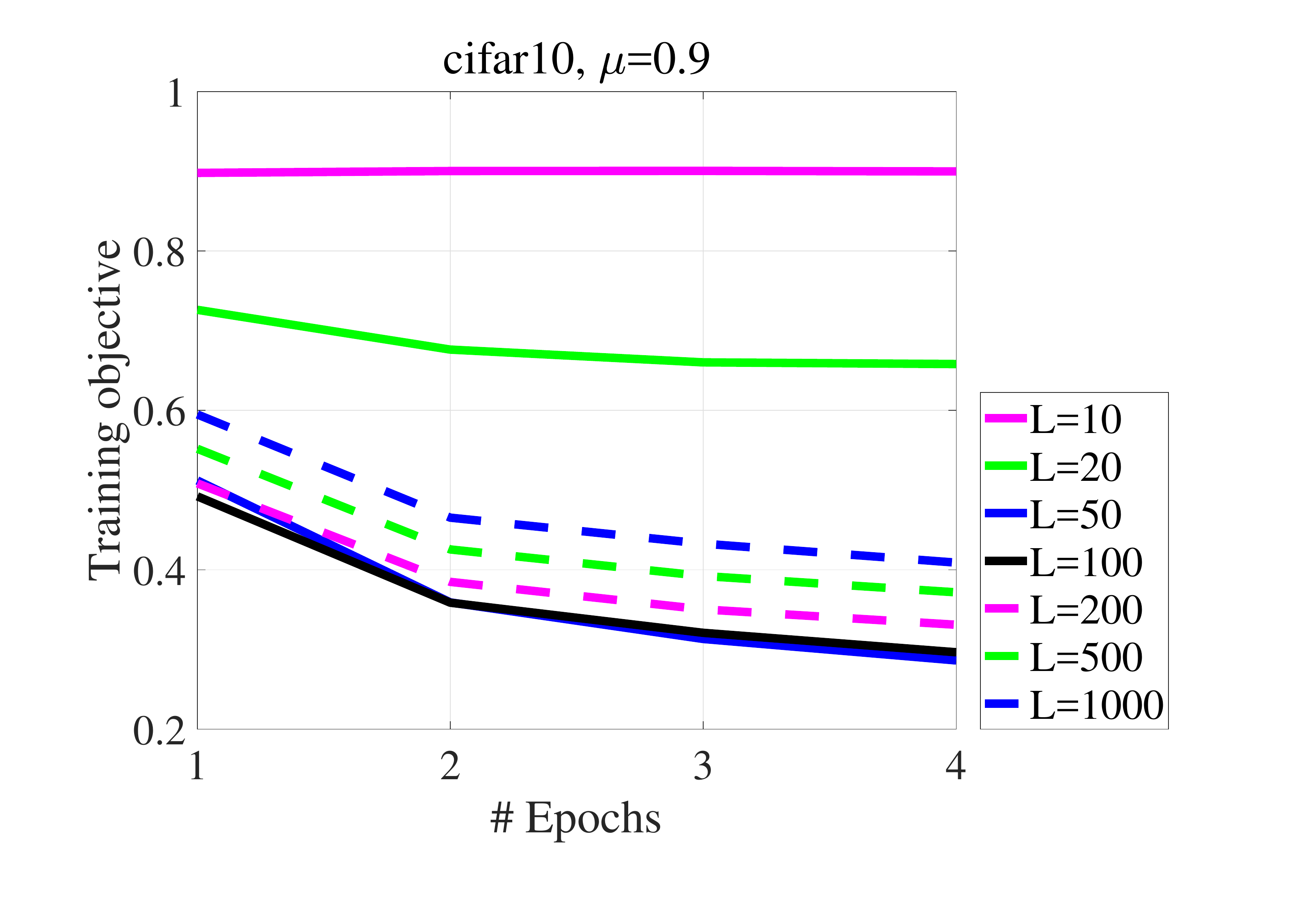}}
		\end{center}
	\end{minipage}
	\vspace{-10mm}
	\caption{\footnotesize Illustration of robustness of Lipschitz constant $L$ in our solver.}\label{fig:L}
\end{figure}

\begin{figure}[t]
	\begin{minipage}[b]{0.495\columnwidth}
		\begin{center}
			\centerline{\includegraphics[width=\columnwidth,height=0.77\columnwidth]{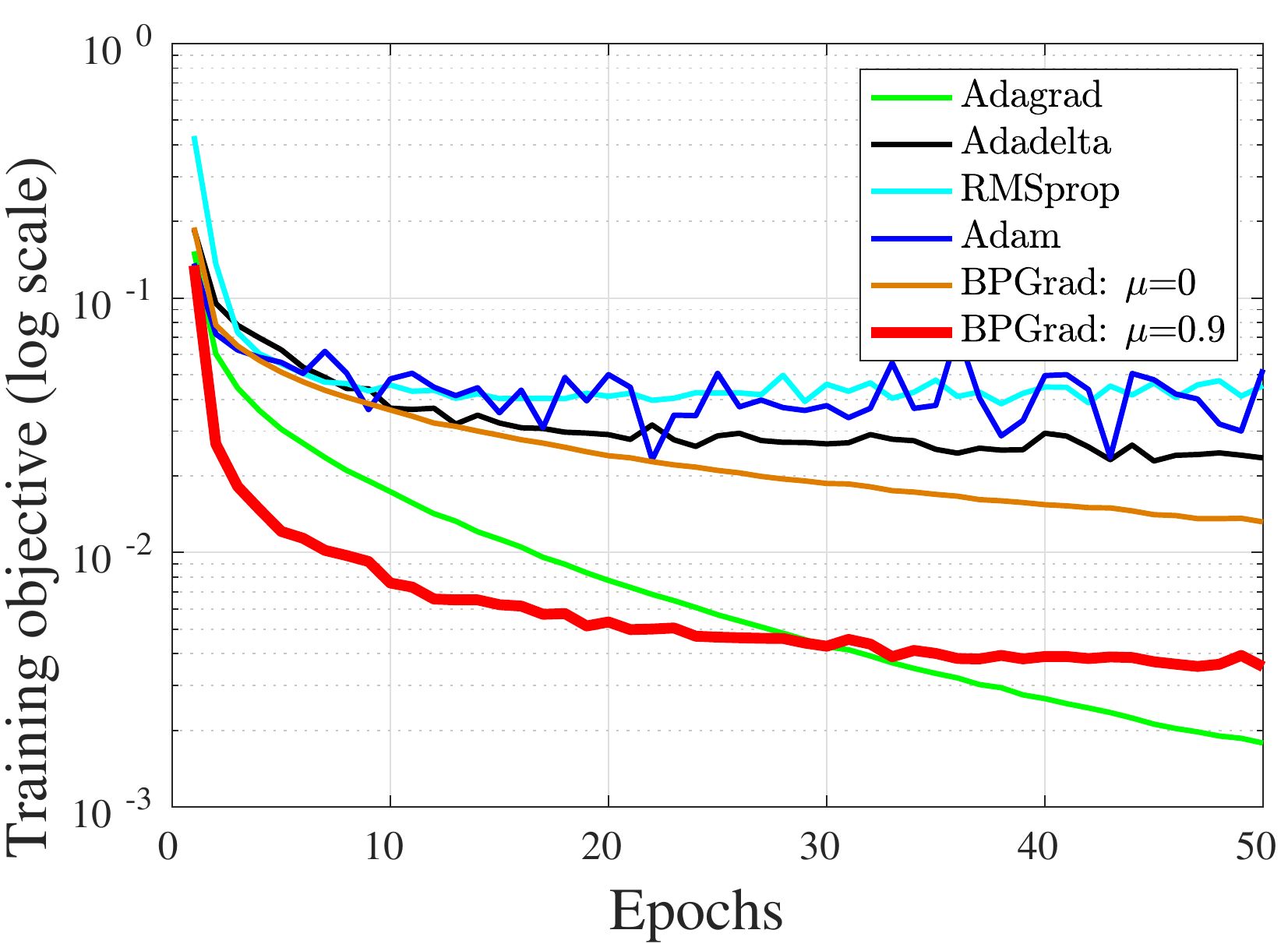}}			
		\end{center}
	\end{minipage}
	\begin{minipage}[b]{0.495\columnwidth}
		\begin{center}
			\centerline{\includegraphics[width=\columnwidth,height=0.75\columnwidth]{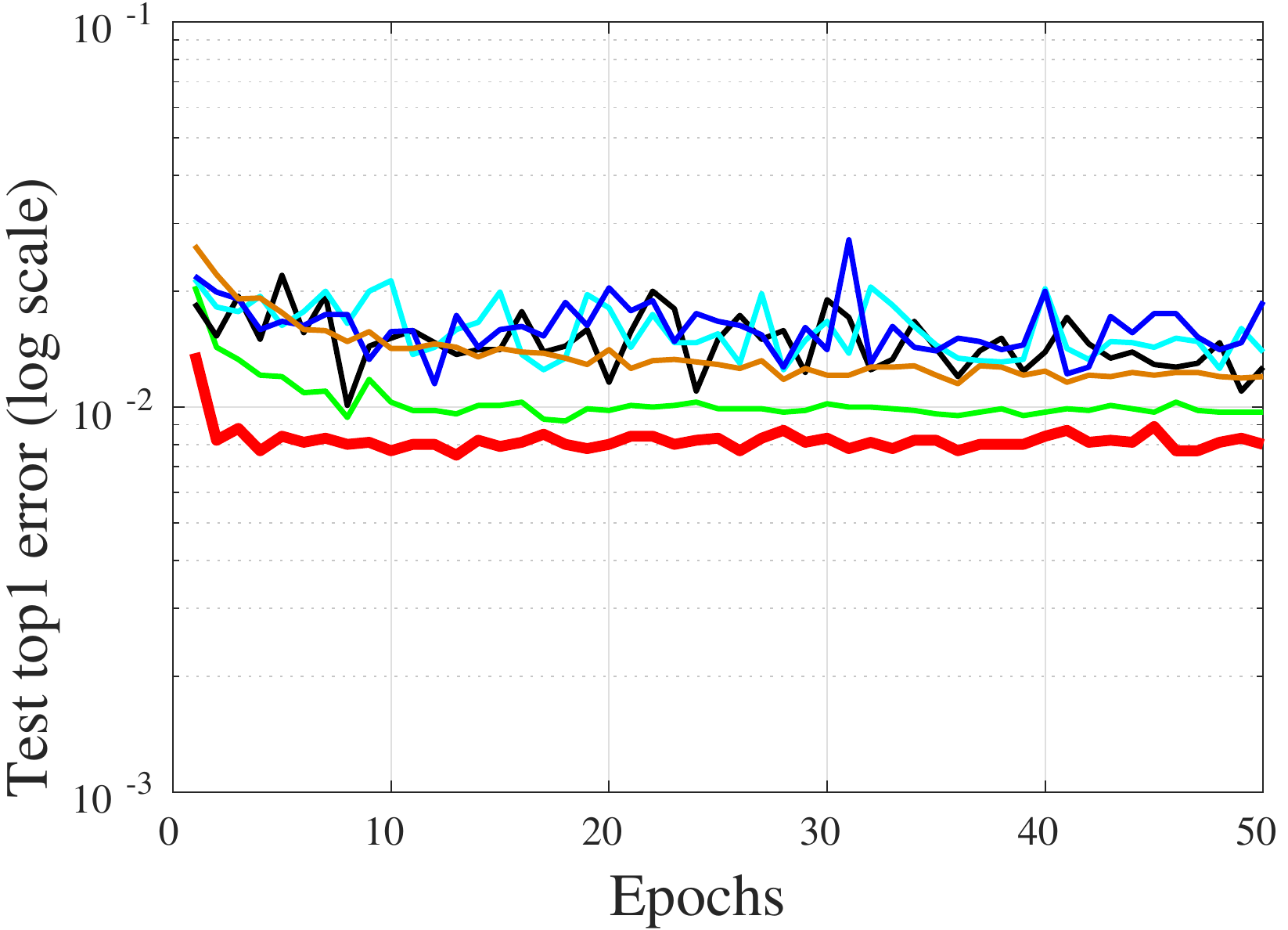}}
		\end{center}
	\end{minipage}	
	\begin{minipage}[b]{0.495\columnwidth}
		\begin{center}			\centerline{\includegraphics[width=\columnwidth,height=0.75\columnwidth]{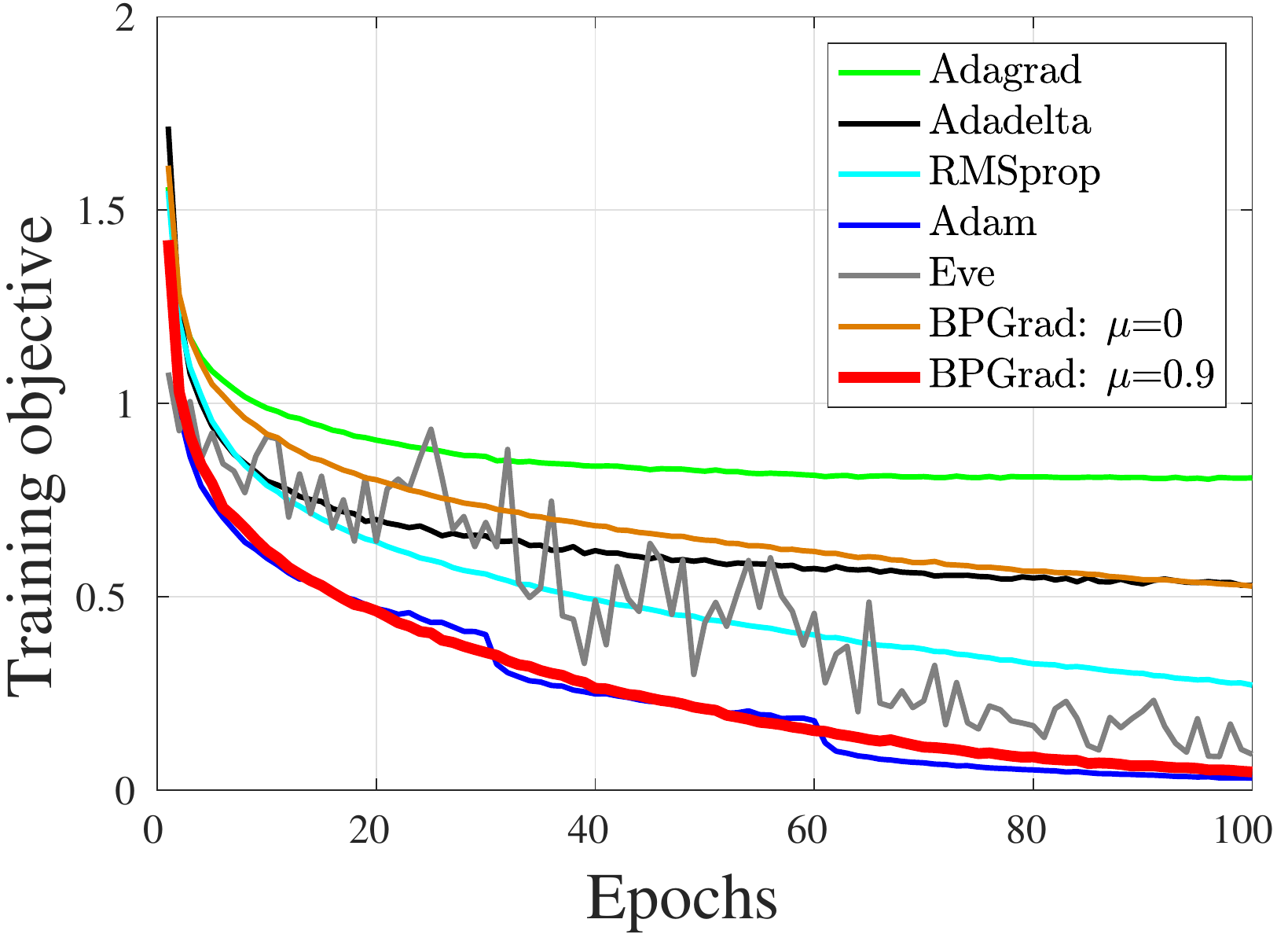}}			
		\end{center}
	\end{minipage}
	\begin{minipage}[b]{0.495\columnwidth}
		\begin{center}
			\centerline{\includegraphics[width=\columnwidth,height=0.75\columnwidth]{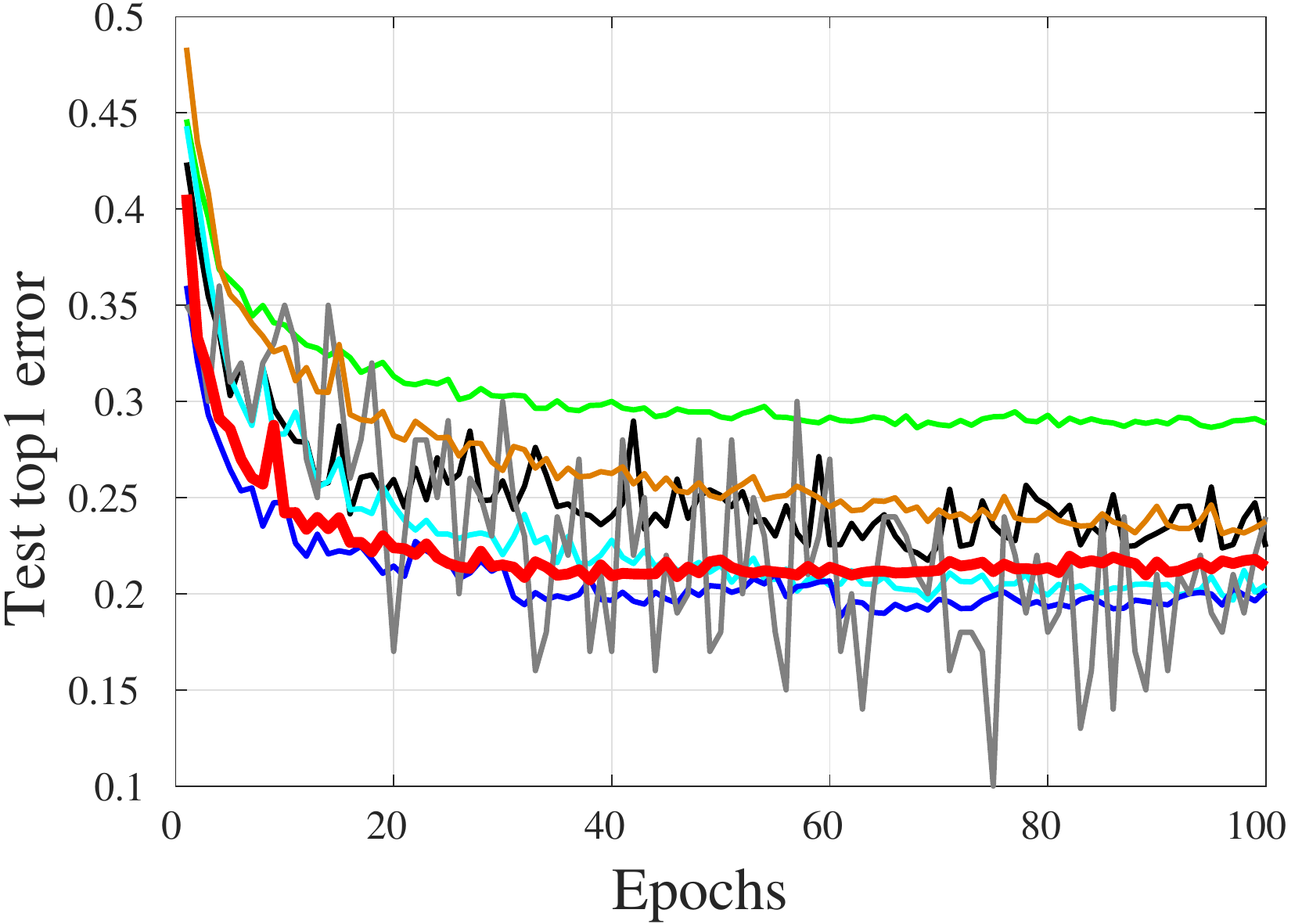}}
		\end{center}
	\end{minipage}		
	\vspace{-7mm}
	\caption{\footnotesize Comparison on {\bf (left)} training objectives and {\bf (right)} test top-1 errors for object recognition using {\bf (top)} MNIST and {\bf (bottom)} CIFAR10.}\label{fig:classification}
    \vspace{-3mm}
\end{figure}

For our solver, the parameters $n$ and $N$ typically depend on the numbers of mini-batches and epochs, respectively. Empirically we find that $N=1$ seems to work well, and thus we use it by default for all the experiments. Accordingly by default $n$ will be set to the product of the numbers of mini-batches and epochs. 

Also we find that the parameter $L$ as Lipschitz constant is quite robust \wrt performance, indicating that heavily tuning this parameter is unnecessary in practice. To demonstrate this, we compare the training objectives of our solver by varying $L$ in Fig.~\ref{fig:L}. To highlight the differences, here we crop and show the results in the first four epochs, but note that the remaining results have similar behavior. As we can see on MNIST when $L$ varies from 10 to 100, the corresponding curves are clustered, similarly on CIFAR10 for $L$ from 50 to 1000. We decide to set $L=15$ for MNIST and $L=50$ for CIFAR10, respectively, in our solver.

Next we show the solver comparison results in Fig. \ref{fig:classification}. To illustrate the effect of momentum in our solver in terms of performance, here we plot two variants of our solver with $\mu=0$ and $\mu=0.9$, respectively. As we see our solver with $\mu=0.9$ works much better than the counterpart, achieving lower training objectives as well as lower top-1 error at test time. This again provides evidence to support the importance of satisfying Eq. \ref{eqn:sampling_rule} in our solver to search for good solutions toward global optimality.

Overall, our solver performs best on MNIST and slightly inferior on CIFAR10 at test time, although in terms of training objective it achieves competitive performance on MNIST and the best on CIFAR10. We hypothesize that this behavior comes from the effect of regularization on Lipschitz continuity. However, our solver can decrease the objectives much faster than all the competitors in the first few epochs. This observation reflects the superior ability of our solver in determining adaptive learning rates for gradients. Especially on CIFAR10 we also compare an extra solver Eve based on our implementation. Eve was proposed in recent related work~\cite{koushik2016improving} that improves Adam with the feedbacks from the objective function, and tested on CIFAR10 as well. As we can see, our solver is much more reliable, performing consistently over epochs.

\subsubsection{ImageNet ILSVRC2012 \cite{krizhevsky2012imagenet}}\label{sub:imagenet}
This dataset contains about $1.28M$ training images and $50K$ validation images among $1000$ object classes. Following the demo code, we train the same AlexNet~\cite{krizhevsky2012imagenet} on it from the scratch using different solvers. We perform training for $20$ epochs, with a mini-batch size $256$, weight decay $0.0005$, momentum $0.9$, and default learning rates for the competitors. For our solver we set $L=100$.

\begin{figure}[t]		
	\begin{minipage}[b]{0.495\columnwidth}
		\begin{center}
			\centerline{\includegraphics[width=\columnwidth,height=0.75\columnwidth]{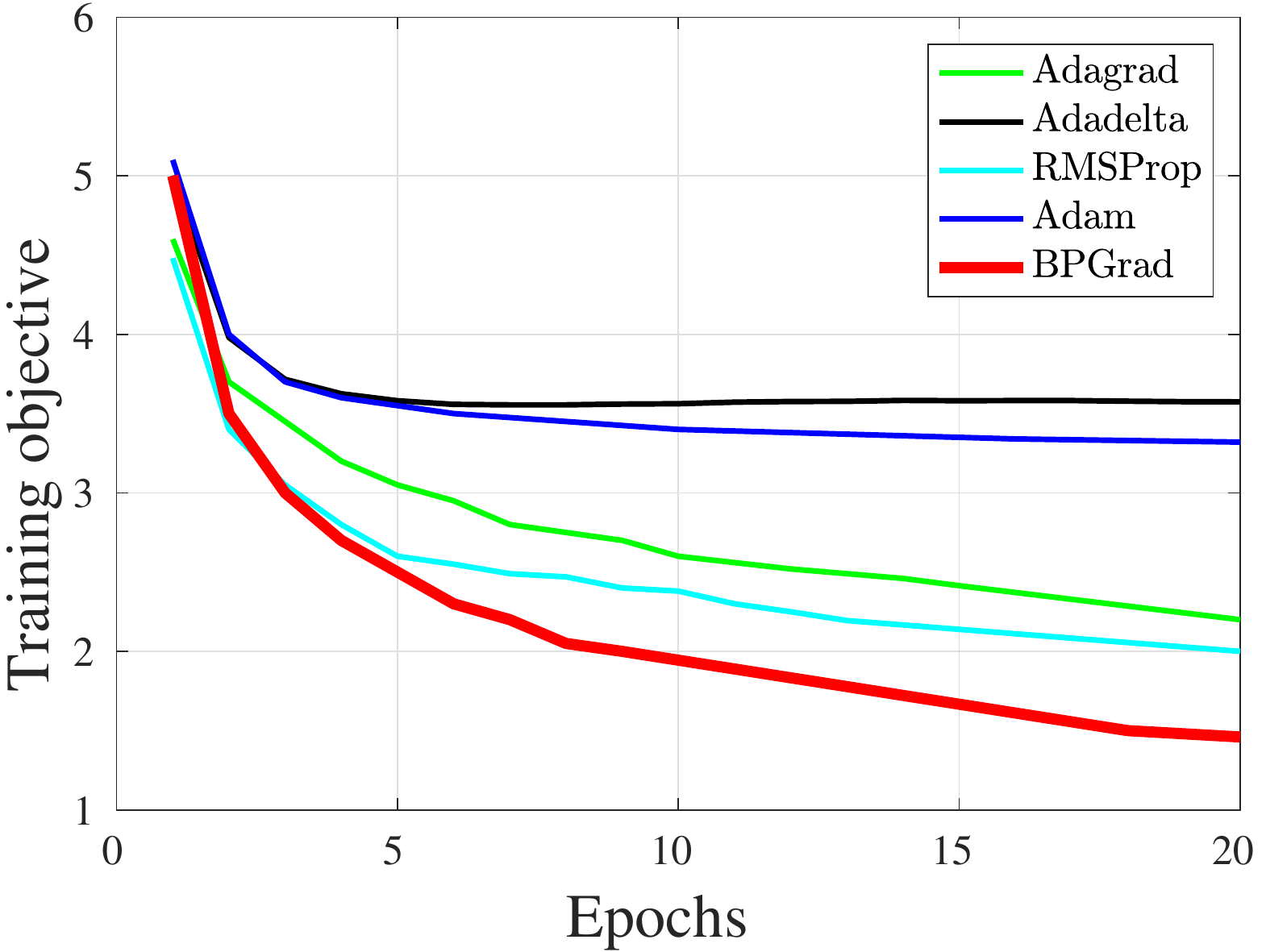}}
		\end{center}
	\end{minipage}	
	\begin{minipage}[b]{0.495\columnwidth}
		\begin{center}
			\centerline{\includegraphics[width=\columnwidth,height=0.75\columnwidth]{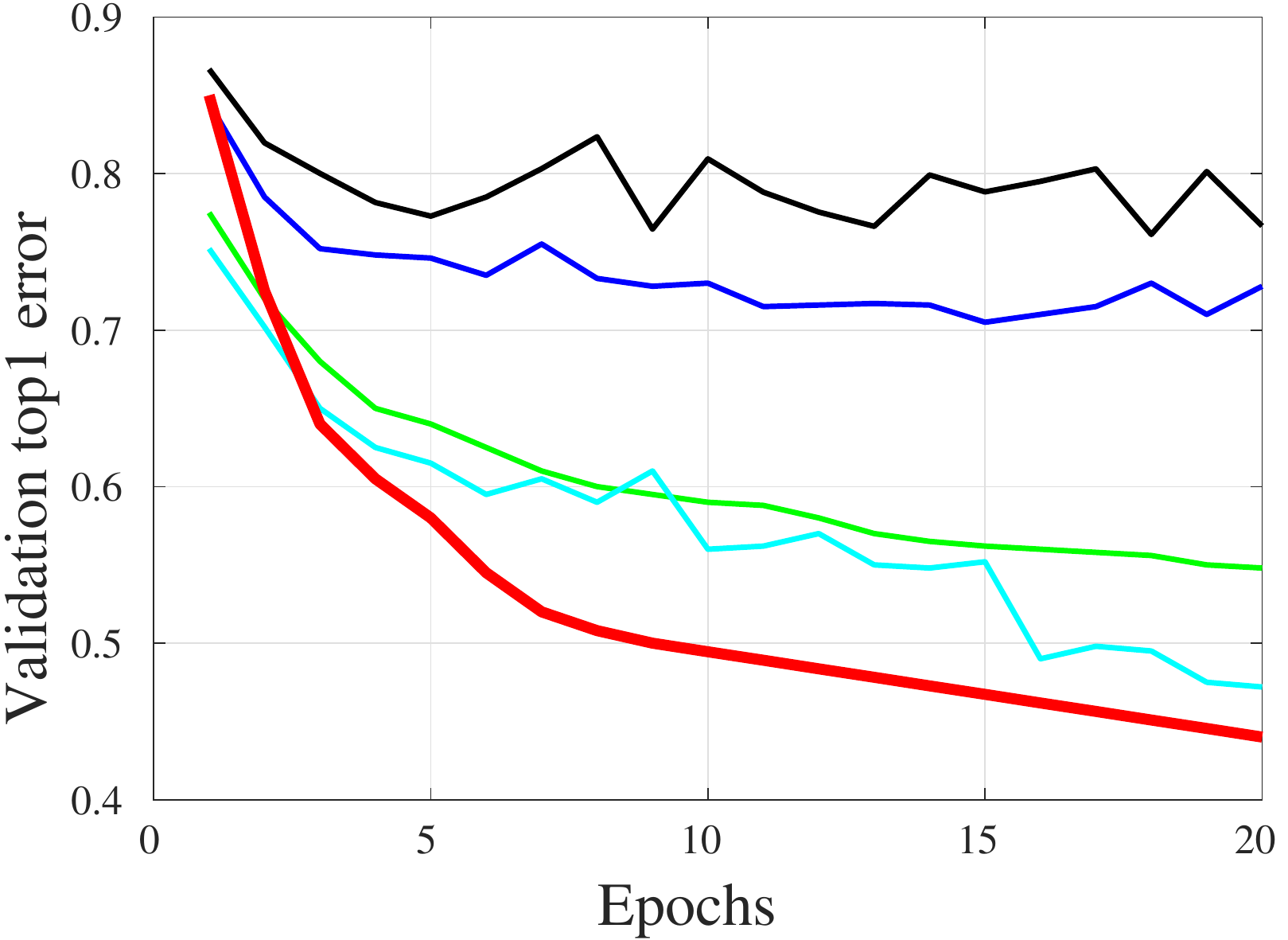}}		
		\end{center}
	\end{minipage}	
	\vspace{-7mm}
	\caption{\footnotesize Comparison on {\bf (left)} training objectives and {\bf (right)} validation top-1 errors for object recognition using ImageNet ILSVRC2012.}\label{fig:imagenet}
\end{figure}
 
\setlength{\tabcolsep}{2.5pt}
\begin{table}[t]\small
	\begin{center}    
		\begin{tabular}{|c|c|c|c|c|c|}
			\hline 
								& Adagrad & Adadelta & RMSProp & Adam & BPGrad  \\ \hline
			\shortstack{training}		& 49.0 & 71.6 & 46.0 & 70.0 & \textbf{33.0} 	\\ \hline
			\shortstack{validation}    & 54.8 & 76.7 & 47.2 & 72.8 & \textbf{44.0}  \\ \hline
		\end{tabular}
	\end{center}
	\caption{\footnotesize Top-1 recognition error ($\%$) on ImageNet ILSVRC2012 dataset.}
	\label{table:imagenet}
    \vspace{-3mm}
\end{table}

\setlength{\tabcolsep}{2.0pt}
\begin{table*}[t]\small
	\begin{center}
	\begin{tabular}{|c|c|c|c|c|c|c|c|c|c|c|c|c|c|c|c|c|c|c|c|c|c|}
    \hline
& aero & bike & bird & boat & bottle & bus & car & cat & chair & cow & table & dog & horse & mbike & persn & plant & sheep & sofa & train & tv & mAP\\ \hline
Adagrad & 67.5 & 71.5 & 60.7 & 47.1 & 28.3 & 72.7 & 76.7 & 77.0 & 34.3 & 70.2 & 64.0 & 72.0 & 74.2 & 69.5 & 64.9 & 28.8 & 57.4 & 60.5 & 73.1 & 61.1  & 61.7    \\ \hline
RMSProp & 69.1 & 75.8 & 61.5 & 47.9 & 30.2 & 74.7 & 77.1 & 79.4 & 33.2 & 71.1 & 66.3 & 74.4 & 76.3 & 69.9 & 65.1 & 28.9 & 62.9 & 62.5 & 73.2 & 60.8  & 63.0    \\ \hline
Adam & 68.9 & \textbf{79.9} & 64.1 & \textbf{56.6} & 37.0 & \textbf{77.4} & \textbf{77.7} & 82.5 & 38.2 & 71.5 & 64.7 & \textbf{77.6} & 77.7 & \textbf{75.0} & \textbf{66.8} & 30.6 & \textbf{65.9} & 65.1 & \textbf{74.4} & \textbf{67.9}  & \textbf{66.0} \\ \hline
BPGrad & \textbf{69.4} & 77.7 & \textbf{66.4} & 55.1 & \textbf{37.2} & 76.1 & \textbf{77.7} & \textbf{83.6} & \textbf{38.6} & \textbf{73.8} & \textbf{67.4} & 76.0 & \textbf{81.9} & 72.7 & 66.3 & \textbf{31.0} & 64.2 &\textbf{66.2} & 73.8 & 64.9 & \textbf{66.0} \\ \hline
	\end{tabular}
	\end{center}
	\caption{\footnotesize Average precision (AP, \%) of object detection on VOC2007 test dataset.}
	\label{table:detection_fast_rcnn}
    \vspace{-3mm}
\end{table*}

We show the comparison results in Fig.~\ref{fig:imagenet}. It is evident that our solver works the best at both training and test time. Namely, it converges faster to achieve lower objective as well as lower top-1 error on validation dataset. In terms of numbers, ours is 3.2\% lower than the second best, RMSProp, at the $20$-th epoch as listed in Table \ref{table:imagenet}. 

Based on all the experiments above we conclude that our solver is suitable to train deep models for object recognition.

\subsection{Object Detection}\label{sub:detection}

\begin{figure}[t]
	\begin{minipage}[b]{0.495\columnwidth}
		\begin{center}
			\centerline{\includegraphics[width=\columnwidth,height=0.75\columnwidth]{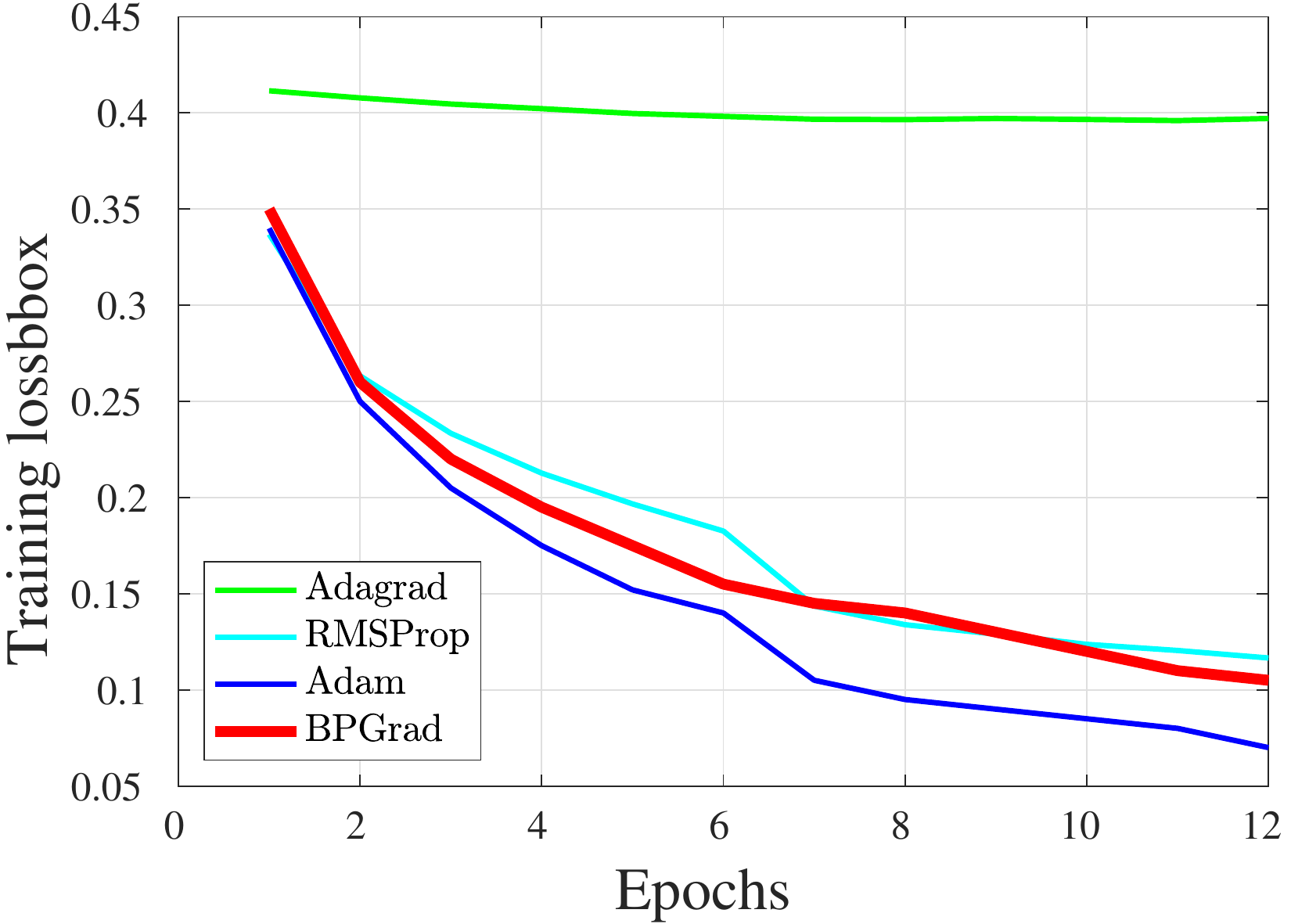}}
		\end{center}
	\end{minipage}	
	\begin{minipage}[b]{0.495\columnwidth}
		\begin{center}
			\centerline{\includegraphics[width=\columnwidth,height=0.75\columnwidth]{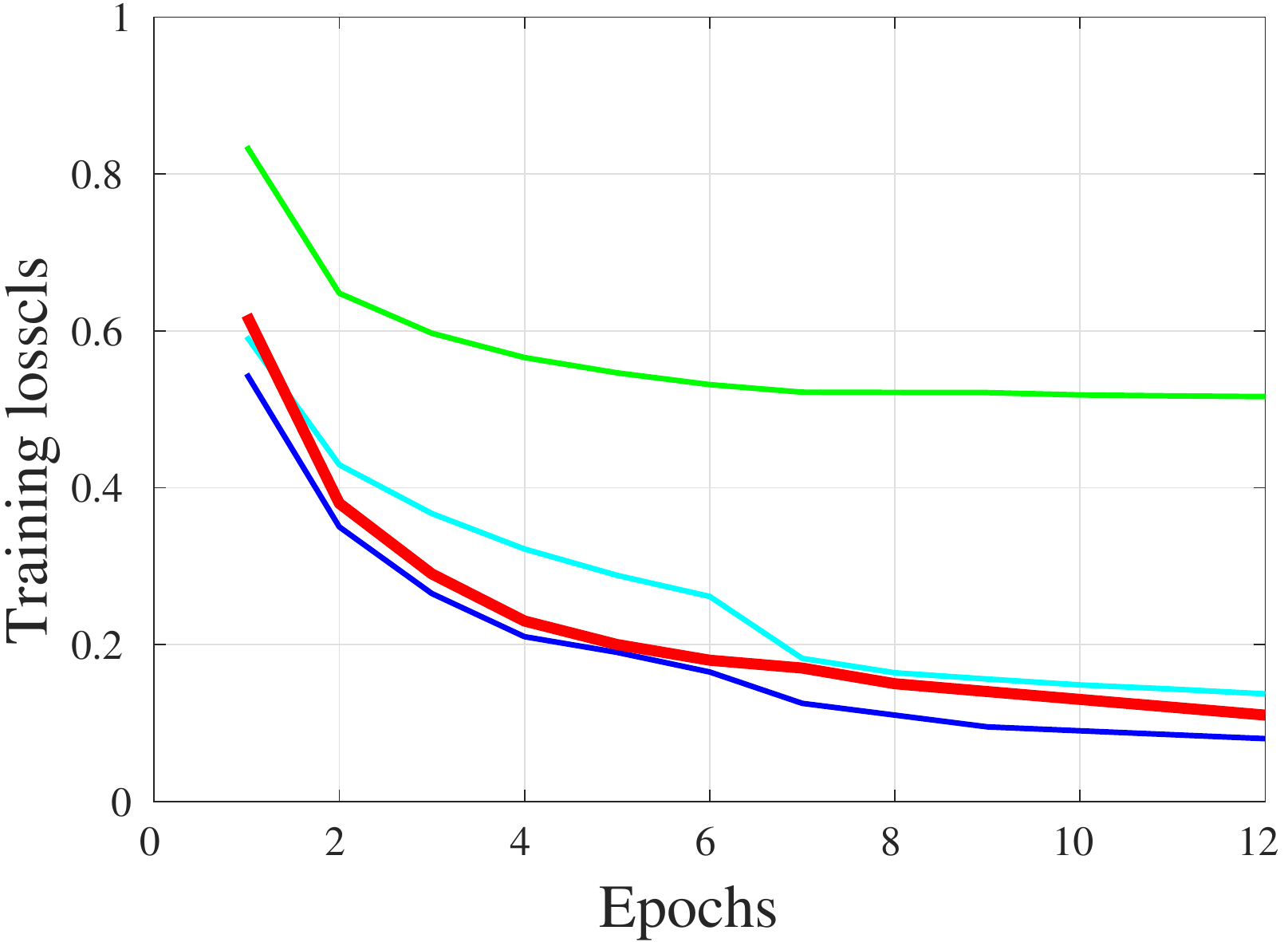}}			
		\end{center}
	\end{minipage}
	\vspace{-7mm}		
	\caption{\footnotesize Loss comparison on VOC2007 trainval dataset, including {\bf (left)} the regression loss using bounding boxes and {\bf (right)} the classification loss.}\label{fig:fast_rcnn_loss}
    \vspace{-3mm}
\end{figure}


Following Fast RCNN~\cite{girshick2015fast} in the demo code, we conduct the solver comparison on the PASCAL VOC2007 dataset~\cite{pascal-voc-2007} with 20 object classes using selective search \cite{uijlings2013selective} as default object proposal approach. For all solvers, we train the network for $12$ epochs using the $5K$ images in VOC2007 trainval set and test it using $4.9K$ images in VOC2007 test set. We set the weight decay and momentum to $0.0005$ and $0.9$, respectively, and use default learning rates for the competitors. We do not compare with Adadelta because we cannot obtain reasonable performance after heavy parameter tuning. For our solver we set $L=100$ and $N=12$.

We show the training comparison in Fig.~\ref{fig:fast_rcnn_loss}, and test results in Table \ref{table:detection_fast_rcnn}. Though our training losses are inferior to those of Adam in this case, our solver works as well as Adam at test time on average, achieving best AP on $11$ out of $20$ classes. This demonstrates the suitability of our solver in training deep models for object detection.

\subsection{Object Segmentation}\label{sub:segmentation}

Following the work~\cite{long2015fully} for semantic segmentation based on fully convolutional networks (FCN), we train FCN-32s with per-pixel multinomial logistic loss and validate it with the standard metric of mean pixel intersection over union (IU), pixel accuracy, and mean accuracy. For all the solvers, we conduct training for $50$ epochs with momentum $0$ and weight decay $0.0005$ on PASCAL VOC2011 \cite{pascal-voc-2011} segmentation set. For Adagrad, RMSProp and Adam, we find that the default parameters are able to achieve the best performance. For Adadelta, we tune its parameters with $\epsilon=10^{-9}$. The global learning rate for RMSProp is set to $10^{-5}$ and $10^{-4}$ for both Adagrad and Adam. Adadelta does not require the global learning rate. For our solver, we set $L=500$.  

We show the learning curves on training and validation datasets in  Fig.~\ref{fig:fcn_segmentation}, and list the test-time comparison results in Table \ref{table:fcn_segmentation}. In this case our solver has very similar learning behavior as Adagrad, but achieves the best performance at test time. The smaller fluctuation over epochs on the validation dataset demonstrates again the superior reliability of our solver, compared with the competitors. Taking these observations into account, we believe that our solver has the ability of learning robust deep models for object segmentation.

\begin{figure}[t]
	\begin{minipage}[b]{0.495\columnwidth}
		\begin{center}
			\centerline{\includegraphics[width=\columnwidth,height=0.75\columnwidth]{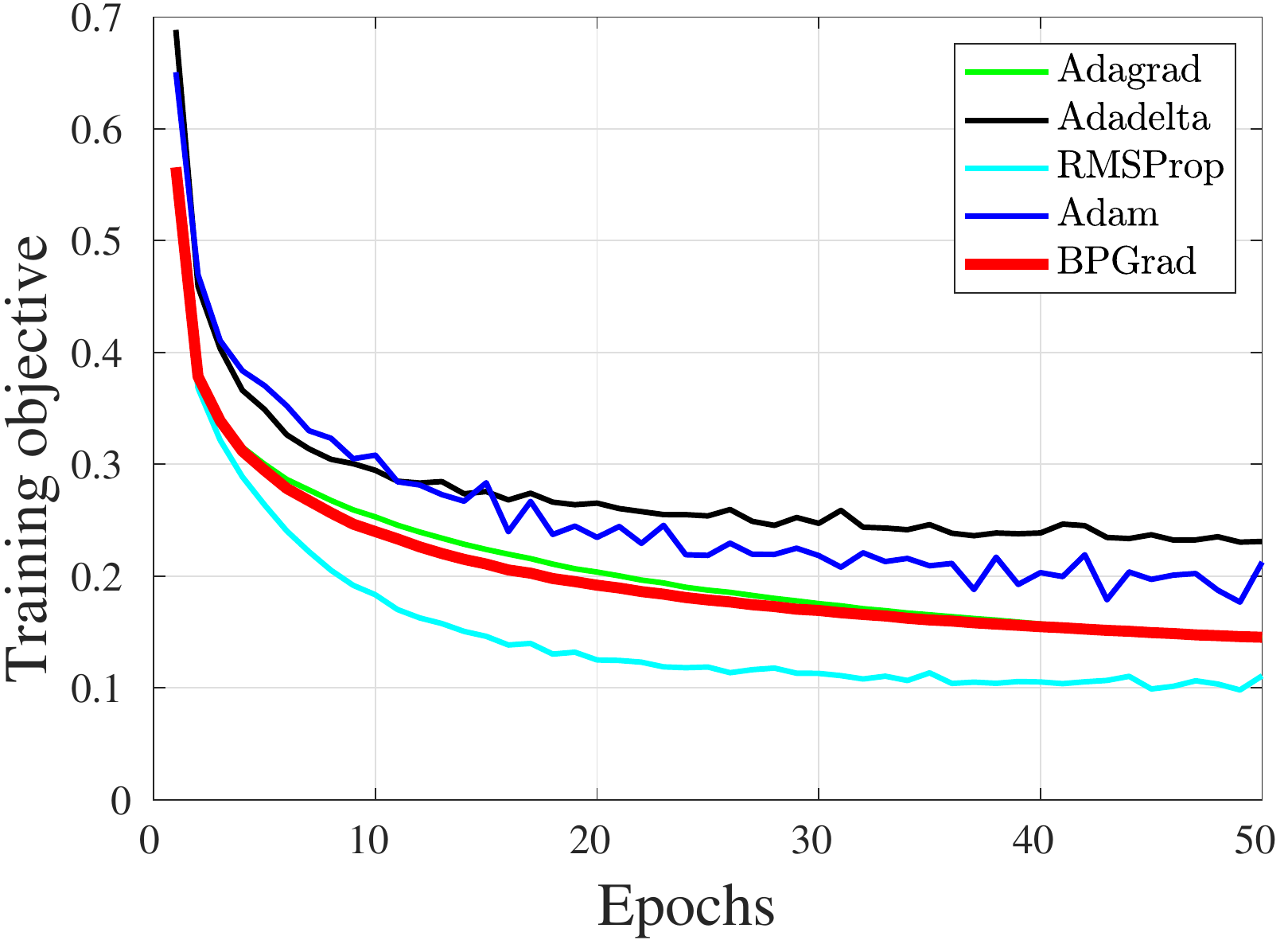}}
		\end{center}
	\end{minipage}	
	\begin{minipage}[b]{0.495\columnwidth}
		\begin{center}
			\centerline{\includegraphics[width=\columnwidth,height=0.75\columnwidth]{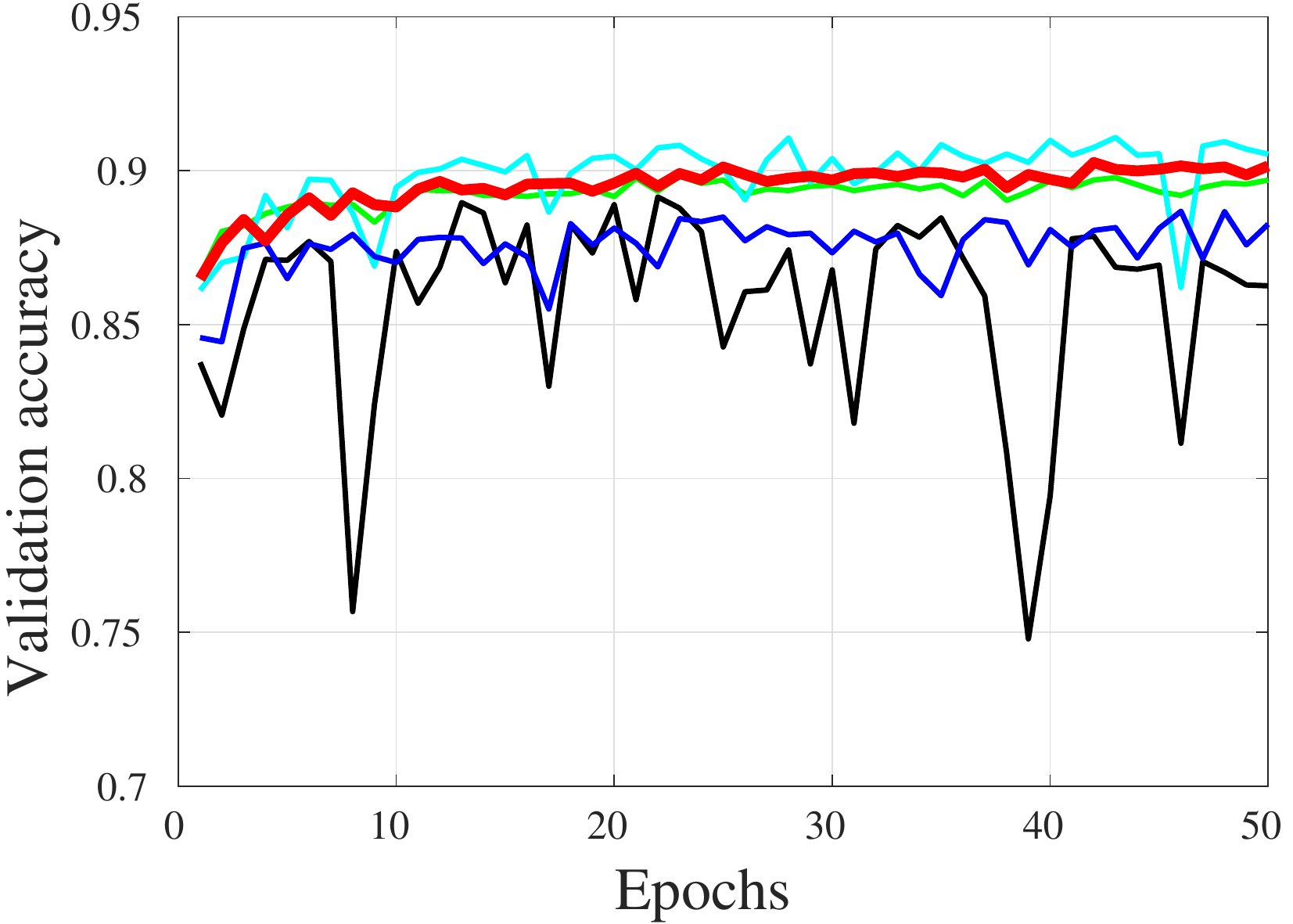}}			
		\end{center}
	\end{minipage}	
    \vspace{-7mm}
	\caption{\footnotesize Segmentation performance comparison using FCN-32s model on VOC2011 training and validation datasets.}\label{fig:fcn_segmentation}
\end{figure}
 
\setlength{\tabcolsep}{1.5pt}
\begin{table}[t]\small
	\begin{center}	
	\begin{tabular}{|c|c|c|c|c|}
		\hline
& mean IU    & pixel accuracy   & mean accuracy & average \\ \hline
 Adagrad   & 60.8		& 89.5    			& 77.4   & 75.9   \\ \hline
 Adadelta  & 46.6     				& 86.0    			& 54.4   & 62.3   \\ \hline
 RMSProp   & 60.5     				& \textbf{90.2}    & 71.0   & 73.9    \\ \hline
 Adam 	   & 50.9					& 87.2   			& 66.4    & 68.2   \\ \hline
 BPGrad    & \textbf{62.4}     	& 89.8 & \textbf{79.6}  & {\bf 77.3}	  \\  \hline
 	\end{tabular}
    \end{center}	
    \caption{\footnotesize Numerical comparison on semantic segmentation performance (\%) using VOC2011 test dataset at the $50$-th epoch.}
	\label{table:fcn_segmentation}
    \vspace{-3mm}
\end{table}

\section{Conclusion}
In this paper we propose a novel approximation algorithm, namely BPGrad, towards searching for global optimality in DL via branch and pruning based on Lipschitz continuity assumption. Our basic idea is to keep generating new samples from the parameter space (\ie branch) outside the removable parameter space (\ie pruning). Lipschitz continuity not only provides us a way to estimate the lower and upper bounds of global optimality, but also serves as regularization to further smooth the objective functions in DL. Theoretically we prove that under some conditions our BPGrad algorithm can converge to global optimality within finite iterations. Empirically in order to avoid the high demand of computation as well as storage for BPGrad in DL, we propose a new efficient solver. Theoretical and empirical justification on preserving the properties of BPGrad is provided. We demonstrate the superiority of our solver to several conventional DL solvers in object recognition, detection, and segmentation.

\newpage
{\footnotesize
	\bibliographystyle{ieee}
	\bibliography{egbib}

\begin{thebibliography}{10}\itemsep=-1pt

\bibitem{blum1989training}
A.~Blum and R.~L. Rivest.
\newblock Training a 3-node neural network is np-complete.
\newblock In {\em NIPS}, pages 494--501, 1989.

\bibitem{bottou2016optimization}
L.~Bottou, F.~E. Curtis, and J.~Nocedal.
\newblock Optimization methods for large-scale machine learning.
\newblock {\em arXiv preprint arXiv:1606.04838}, 2016.

\bibitem{brutzkus2017globally}
A.~Brutzkus and A.~Globerson.
\newblock Globally optimal gradient descent for a convnet with gaussian inputs.
\newblock {\em arXiv preprint arXiv:1702.07966}, 2017.

\bibitem{chaudhari2016entropy}
P.~Chaudhari, A.~Choromanska, S.~Soatto, and Y.~LeCun.
\newblock Entropy-sgd: Biasing gradient descent into wide valleys.
\newblock {\em arXiv preprint arXiv:1611.01838}, 2016.

\bibitem{choromanska2015loss}
A.~Choromanska, M.~Henaff, M.~Mathieu, G.~B. Arous, and Y.~LeCun.
\newblock The loss surfaces of multilayer networks.
\newblock In {\em AISTATS}, pages 192--204, 2015.

\bibitem{imagenet_cvpr09}
J.~Deng, W.~Dong, R.~Socher, L.-J. Li, K.~Li, and L.~Fei-Fei.
\newblock {ImageNet: A Large-Scale Hierarchical Image Database}.
\newblock In {\em CVPR}, 2009.

\bibitem{duchi2011adaptive}
J.~Duchi, E.~Hazan, and Y.~Singer.
\newblock Adaptive subgradient methods for online learning and stochastic
  optimization.
\newblock {\em JMLR}, 12(Jul):2121--2159, 2011.

\bibitem{erikssonapplied}
K.~Eriksson, D.~Estep, and C.~Johnson.
\newblock {\em Applied Mathematics Body and Soul: Vol I-III}.
\newblock Springer-Verlag Publishing, 2003.

\bibitem{pascal-voc-2007}
M.~Everingham, L.~Van~Gool, C.~K.~I. Williams, J.~Winn, and A.~Zisserman.
\newblock The {PASCAL} {V}isual {O}bject {C}lasses {C}hallenge 2007 {(VOC2007)}
  {R}esults.
\newblock
  http://www.pascal-network.org/challenges/VOC/voc2007/workshop/index.html.

\bibitem{pascal-voc-2011}
M.~Everingham, L.~Van~Gool, C.~K.~I. Williams, J.~Winn, and A.~Zisserman.
\newblock The {PASCAL} {V}isual {O}bject {C}lasses {C}hallenge 2011 {(VOC2011)}
  {R}esults.
\newblock
  http://www.pascal-network.org/challenges/VOC/voc2011/workshop/index.html.

\bibitem{girshick2015fast}
R.~Girshick.
\newblock Fast r-cnn.
\newblock In {\em CVPR}, pages 1440--1448, 2015.

\bibitem{goyal2017accurate}
P.~Goyal, P.~Doll{\'a}r, R.~Girshick, P.~Noordhuis, L.~Wesolowski, A.~Kyrola,
  A.~Tulloch, Y.~Jia, and K.~He.
\newblock Accurate, large minibatch sgd: Training imagenet in 1 hour.
\newblock {\em arXiv preprint arXiv:1706.02677}, 2017.

\bibitem{haeffele2017global}
B.~D. Haeffele and R.~Vidal.
\newblock Global optimality in neural network training.
\newblock In {\em CVPR}, pages 7331--7339, 2017.

\bibitem{hand2017global}
P.~Hand and V.~Voroninski.
\newblock Global guarantees for enforcing deep generative priors by empirical
  risk.
\newblock {\em arXiv preprint arXiv:1705.07576}, 2017.

\bibitem{he2016deep}
K.~He, X.~Zhang, S.~Ren, and J.~Sun.
\newblock Deep residual learning for image recognition.
\newblock In {\em CVPR}, pages 770--778, 2016.

\bibitem{hinton2012deep}
G.~Hinton, L.~Deng, D.~Yu, G.~E. Dahl, A.-r. Mohamed, N.~Jaitly, A.~Senior,
  V.~Vanhoucke, P.~Nguyen, T.~N. Sainath, et~al.
\newblock Deep neural networks for acoustic modeling in speech recognition: The
  shared views of four research groups.
\newblock {\em IEEE Signal Processing Magazine}, 29(6):82--97, 2012.

\bibitem{kawaguchi2016deep}
K.~Kawaguchi.
\newblock Deep learning without poor local minima.
\newblock In {\em NIPS}, pages 586--594, 2016.

\bibitem{kingma2014adam}
D.~Kingma and J.~Ba.
\newblock Adam: A method for stochastic optimization.
\newblock {\em arXiv preprint arXiv:1412.6980}, 2014.

\bibitem{koushik2016improving}
J.~Koushik and H.~Hayashi.
\newblock Improving stochastic gradient descent with feedback.
\newblock {\em arXiv preprint arXiv:1611.01505}, 2016.

\bibitem{krizhevsky2012imagenet}
A.~Krizhevsky, I.~Sutskever, and G.~E. Hinton.
\newblock Imagenet classification with deep convolutional neural networks.
\newblock In {\em NIPS}, pages 1097--1105, 2012.

\bibitem{lecun1998mnist}
Y.~LeCun.
\newblock The mnist database of handwritten digits.
\newblock \url{http://yann.lecun.com/exdb/mnist/}, 1998.

\bibitem{lecun1998gradient}
Y.~LeCun, L.~Bottou, Y.~Bengio, and P.~Haffner.
\newblock Gradient-based learning applied to document recognition.
\newblock {\em Proceedings of the IEEE}, 86(11):2278--2324, 1998.

\bibitem{pmlr-v49-lee16}
J.~D. Lee, M.~Simchowitz, M.~I. Jordan, and B.~Recht.
\newblock Gradient descent only converges to minimizers.
\newblock In {\em COLT}, pages 1246--1257, 2016.

\bibitem{long2015fully}
J.~Long, E.~Shelhamer, and T.~Darrell.
\newblock Fully convolutional networks for semantic segmentation.
\newblock In {\em CVPR}, pages 3431--3440, 2015.

\bibitem{MalherbeICML17}
C.~Malherbe and N.~Vayatis.
\newblock Global optimization of lipschitz functions.
\newblock In {\em ICML}, 2017.

\bibitem{mukkamala2017variants}
M.~C. Mukkamala and M.~Hein.
\newblock Variants of rmsprop and adagrad with logarithmic regret bounds.
\newblock {\em arXiv preprint arXiv:1706.05507}, 2017.

\bibitem{nguyen2017loss}
Q.~Nguyen and M.~Hein.
\newblock The loss surface of deep and wide neural networks.
\newblock {\em arXiv preprint arXiv:1704.08045}, 2017.

\bibitem{panageas2016gradient}
I.~Panageas and G.~Piliouras.
\newblock Gradient descent only converges to minimizers: Non-isolated critical
  points and invariant regions.
\newblock {\em arXiv preprint arXiv:1605.00405}, 2016.

\bibitem{sotiropoulos2001branch}
D.~G. Sotiropoulos and T.~N. Grapsa.
\newblock A branch-and-prune method for global optimization.
\newblock In {\em Scientific Computing, Validated Numerics, Interval Methods},
  pages 215--226. Springer, 2001.

\bibitem{soudry2016no}
D.~Soudry and Y.~Carmon.
\newblock No bad local minima: Data independent training error guarantees for
  multilayer neural networks.
\newblock {\em arXiv preprint arXiv:1605.08361}, 2016.

\bibitem{sutskever2013importance}
I.~Sutskever, J.~Martens, G.~Dahl, and G.~Hinton.
\newblock On the importance of initialization and momentum in deep learning.
\newblock In {\em ICML}, pages 1139--1147, 2013.

\bibitem{sutskever2014sequence}
I.~Sutskever, O.~Vinyals, and Q.~V. Le.
\newblock Sequence to sequence learning with neural networks.
\newblock In {\em NIPS}, pages 3104--3112, 2014.

\bibitem{Tieleman2012}
T.~Tieleman and G.~Hinton.
\newblock {Lecture 6.5---RmsProp: Divide the gradient by a running average of
  its recent magnitude}.
\newblock COURSERA: Neural Networks for Machine Learning, 2012.

\bibitem{uijlings2013selective}
J.~R. Uijlings, K.~E. Van De~Sande, T.~Gevers, and A.~W. Smeulders.
\newblock Selective search for object recognition.
\newblock {\em IJCV}, 104(2):154--171, 2013.

\bibitem{vedaldi2015matconvnet}
A.~Vedaldi and K.~Lenc.
\newblock Matconvnet: Convolutional neural networks for matlab.
\newblock In {\em ACM Multimedia}, pages 689--692, 2015.

\bibitem{yun2017global}
C.~Yun, S.~Sra, and A.~Jadbabaie.
\newblock Global optimality conditions for deep neural networks.
\newblock {\em arXiv preprint arXiv:1707.02444}, 2017.

\bibitem{zeiler2012adadelta}
M.~D. Zeiler.
\newblock Adadelta: an adaptive learning rate method.
\newblock {\em arXiv preprint arXiv:1212.5701}, 2012.

\bibitem{zhang2016understanding}
C.~Zhang, S.~Bengio, M.~Hardt, B.~Recht, and O.~Vinyals.
\newblock Understanding deep learning requires rethinking generalization.
\newblock {\em arXiv preprint arXiv:1611.03530}, 2016.

\bibitem{zhang2015deep}
S.~Zhang, A.~E. Choromanska, and Y.~LeCun.
\newblock Deep learning with elastic averaging sgd.
\newblock In {\em NIPS}, pages 685--693, 2015.

\bibitem{zhang2017Convergent}
Z.~Zhang and M.~Brand.
\newblock Convergent block coordinate descent for training tikhonov regularized
  deep neural networks.
\newblock In {\em NIPS}, 2017.

\end{thebibliography}
}

\end{document}